\documentclass[10pt]{article} 
\usepackage[preprint]{tmlr}

\usepackage{algorithm}
\usepackage{algpseudocode}
\usepackage{graphicx} 
\usepackage{subcaption}
\usepackage{multirow}
\usepackage{hyperref}
\usepackage{array}
\usepackage{fullpage}
\usepackage{adjustbox}
\usepackage{lineno}

\usepackage{url}
\usepackage{xcolor}

\newcommand{\old}{\mathrm{old}}
\newcommand{\train}{\mathrm{train}}
\newcommand{\infer}{\mathrm{infer}}
\newcommand{\alg}{Adaptive Layerwise Perturbation}

\usepackage{multirow}
\usepackage{threeparttable}
\usepackage{makecell}
\usepackage{booktabs} 

\usepackage{amsmath,amsfonts,bm}

\def\eqref#1{Equation~(\ref{#1})}

\def\1{\bm{1}}

\RequirePackage{amsmath}
\RequirePackage{amssymb}
\RequirePackage{amsthm}
\RequirePackage{bm} 
\RequirePackage{url}
\usepackage{multirow}
\usepackage{natbib}
\usepackage{graphicx}
\usepackage{makecell}
\usepackage{booktabs}
\usepackage{array}
\usepackage{url}

\usepackage{dsfont}

\usepackage{enumerate}
\usepackage[OT1]{fontenc}
\usepackage{natbib}

\usepackage{mathrsfs}

\usepackage{xcolor}
\usepackage{hyperref}
\usepackage{float}  
\usepackage{colortbl}

\def\##1\#{\begin{align}#1\end{align}}
\def\$#1\${\begin{align*}#1\end{align*}}

\let\tilde\widetilde

\newcommand{\E}{\mathbb{E}}

\usepackage{xcolor}

\definecolor{red1}{HTML}{f47983}
\definecolor{blue1}{HTML}{3eede7}
\definecolor{yellow1}{HTML}{f5dd6f}

\newtheorem{theorem}{Theorem}
\newtheorem{condition}{Condition}

\newtheorem{remark}{Remark}

\newtheorem{lemma}{Lemma}

\title{Adaptive Layerwise Perturbation: Unifying Off-Policy Corrections for LLM RL}

\author{Chenlu Ye\thanks{Equal contribution.
 Correspondence to Chenlu Ye (\url{chenluy3@illinois.edu}).}$\ ^\circ$,\quad Xuanchang Zhang$^{*\circ}$,\quad Yifan Hao$^{*\circ}$,\\
 Zhou Yu$^{\diamondsuit}$,\quad Ziji Zhang$^{\diamondsuit}$,\quad Abhinav Gullapalli$^{\diamondsuit}$,\quad Hao Chen$^{\diamondsuit}$, \quad Jing Huang$^{\diamondsuit}$,\quad Tong Zhang$^\circ$
 \\\\
\textnormal{\emph{$^\circ$University of Illinois Urbana-Champaign
\quad$^\diamondsuit$Amazon}}}

\def\month{MM}  
\def\year{YYYY} 
\def\openreview{\url{https://openreview.net/forum?id=XXXX}} 

\begin{document}

\maketitle

\begin{abstract}
Off-policy problems such as policy staleness and training--inference mismatch have become a major bottleneck for training stability and further exploration in LLM RL. The distribution gap between the inference and updated policies grows because of the techniques to enhance inference efficiency, leading to heavy-tailed importance ratios. Heavy-tailed ratios arise when the policy is locally sharp, which further inflates gradients and can push updates outside the trust region. To address this, we propose Adaptive Layerwise Perturbation (ALP), which injects small learnable perturbations into the input hidden states of each layer during updates and uses the resulting perturbed policy as the numerator of the importance ratio against the unchanged inference policy in the objective. Intuitively, by adding controlled noise to intermediate representations, ALP prevents the updated policy from deviating too sharply from the inference policy and enlarges the policy family to cover inference-time mismatch noise. Hence, the flattened distribution can naturally tighten the gap between the updated and inference policies and reduce the tail of importance ratios, thus maintaining training stability. This is further validated empirically. Experiments on single-turn math and multi-turn tool-integrated reasoning tasks show that ALP not only improves final performance, but also avoids blow-up in the importance-ratio tail and KL spikes during iterative training, along with boosted exploration.
Ablations show that representation-level perturbations across all layers are most effective, substantially outperforming partial-layer and logits-only variants. Codes and training recipes are available at \url{https://github.com/amazon-science/adaptive-layerwise-perturbation}.
\end{abstract}


\begin{figure}[ht]
    \centering
    \begin{subfigure}[t]{0.43\textwidth}
        \centering
        \includegraphics[width=\textwidth]{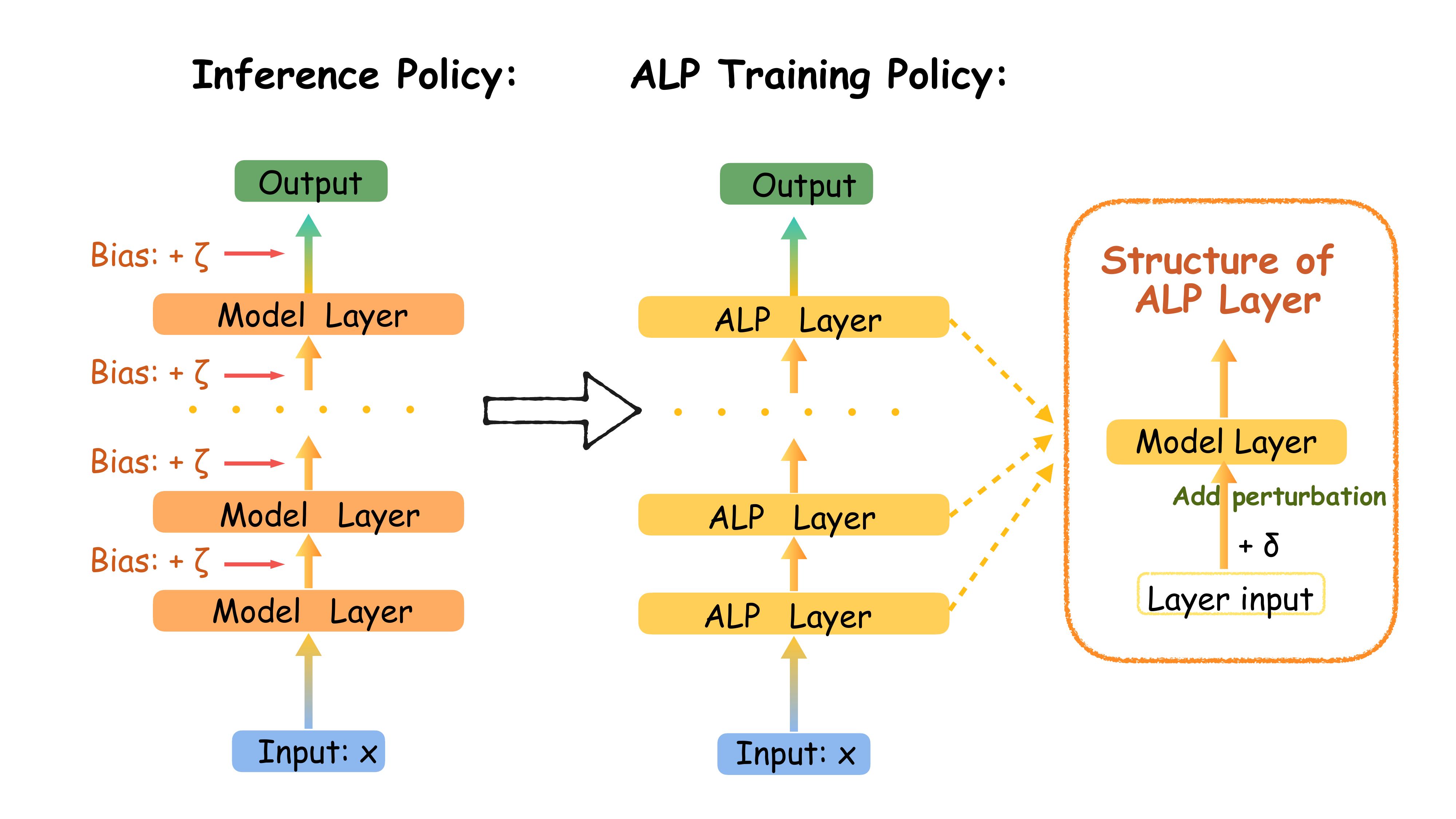}
    \end{subfigure}
        \hfill
    \begin{subfigure}[t]{0.55\textwidth}
        \centering
        \includegraphics[width=\textwidth]{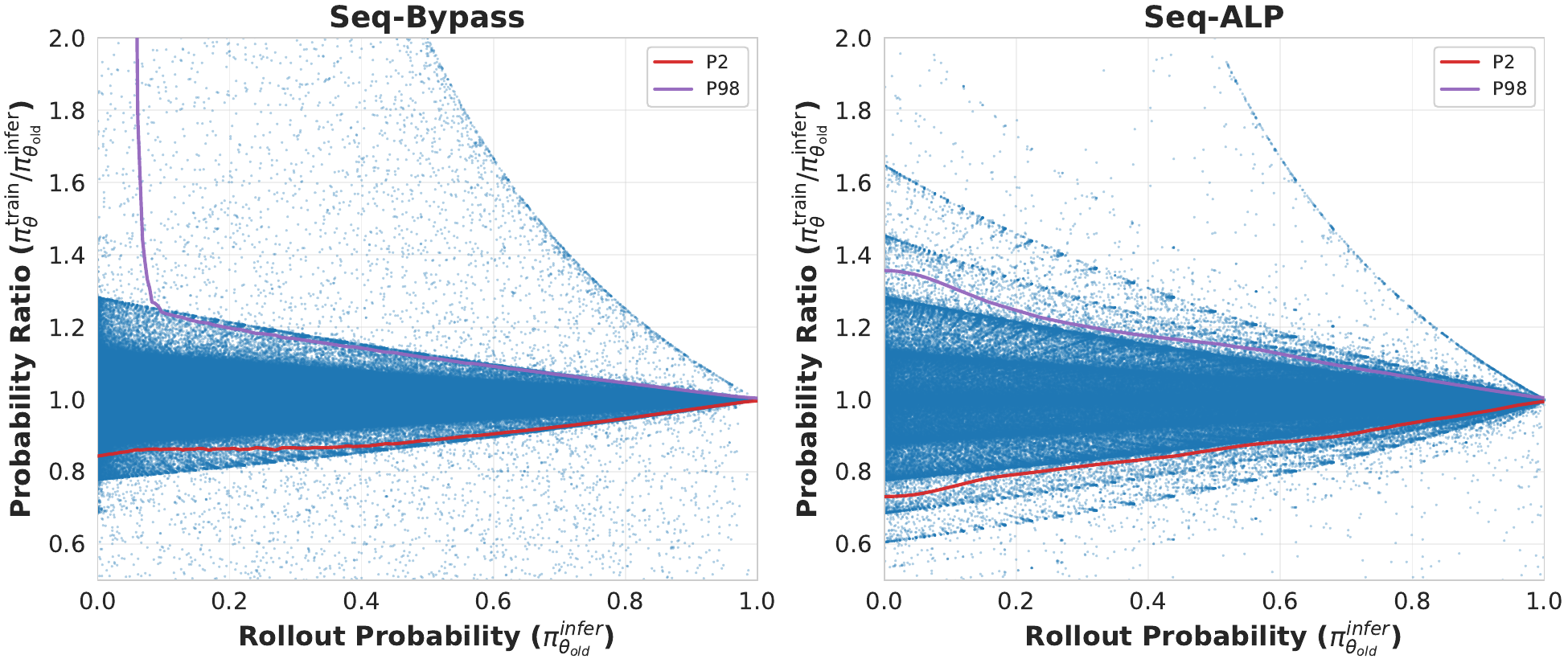}
    \end{subfigure}
    \caption{Left: ALP injects small layerwise perturbations during training to cover training-inference bias. Right: compared with Bypass over $1840$ steps, ALP yields a smoother policy and a tighter importance-ratio envelope, where lines P$2$, P$98$ represent the $2$th and $98$th quantile of the ratio.}
    \label{fig:alp}
\end{figure}

\section{Introduction}

The trade-off between training efficiency and stability is a central challenge in reinforcement learning (RL) for large language models (LLMs). To increase efficiency, practical systems often optimize a policy using rollouts generated by a different distribution, which is known as off-policy. Off-policy effects arise from multiple sources. First, the same batch of trajectories is used across several policy updates. Second, modern inference engines such as vLLM \citep{kwon2023efficient} and SGLang \citep{zheng2024sglang} introduce training–inference mismatch due to quantization, batching, and kernel-level differences, even when weights are nominally identical \citep{yao2025rollouttrainingmismatch}. In multi-turn agentic settings, this mismatch can be further amplified by inference distribution shifts induced by tool feedback and out-of-distribution observations \citep{li_rl_collapse}. These factors lead to heavy-tailed importance ratios, KL spikes, and brittle optimization dynamics, making robust off-policy RL an urgent problem.

A prominent line of work stabilizes training by modifying importance ratios. One approach changes the aggregation way of importance ratios in the objective. It replaces the ratio at each token index in GRPO
\citep{shao2024deepseekmath} with the multiplication of ratios in the whole trajectory \citep{zheng2025group}, but this does not consider training-inference mismatch. Another approach targets on training-inference mismatch. For example, Bypass is a baseline mentioned in \citep{li_rl_collapse,yao2025rollouttrainingmismatch} that uses the ratio of the updated and rollout policies. TIS/MIS \citep{yao2025rollouttrainingmismatch,li_rl_collapse} further correct mismatch by multiplying an auxiliary ratio of the proxy policy (an old policy acting as the anchor) and the rollout policy. Then, they truncate or mask tokens whose auxiliary ratio exceeds a threshold. Although effective at preventing catastrophic collapse, these methods introduce practical limitations: they split off-policy effects into two separately truncated ratios, which can over-truncate updates that are otherwise within a valid trust region, thereby inducing additional bias and leading to early plateau. We defer a detailed discussion of related work to Section~\ref{s:Related Works}.

We argue that off-policy instability in LLM RL is not solely a bookkeeping issue of ``which importance ratio to use'', but also a \textbf{geometry} issue: the noisy update may push the policy toward sharp, brittle regions where small distribution shifts (from staleness or inference mismatch) cause large changes in action probabilities. This motivates a perturbation-sampling perspective from model smoothness and distributional learning approaches \citep{shen2023engression,hao2025towards}: when the dominant failure mode is system-induced noise and staleness, a natural defense is \emph{fighting noise with structured noise}.

Specifically, we propose \textbf{Adaptive Layerwise Perturbation (ALP)} by injecting a learnable Gaussian perturbation \( \delta \sim \mathcal{N}(0, \sigma^2 I) \)\footnote{The perturbation distribution is not restricted to Gaussian; it only needs to have zero mean.} into the input hidden states across layers during policy updates and producing a perturbed training distribution \( \pi^{\mathrm{train}}_{\theta,\sigma} \) (Figure~\ref{fig:alp}, left). We optimize the objective using a single unified ratio of $\pi^{\mathrm{train}}_{\theta,\sigma}$ and the unperturbed inference policy;
i.e., perturbation is applied only to the numerator during training. 

Intuitively, ALP can \textbf{smooth} the local optimization landscape to suppress heavy-tailed ratio excursions and enlarge the training distribution family. This is also evidenced by the full-training comparison in Figure~\ref{fig:alp} (right). Starting from the same base model, the training policy without perturbation (Bypass) gradually becomes sharp and brittle, and the importance-ratio tail eventually explodes. In contrast, the training policy with perturbation (ALP) stays stable while controlling the mismatch between the training and inference systems, thus tightening the ratio-quantile envelope. Beyond stability, by enlarging effective support and preventing premature concentration on brittle modes, ALP also encourages \textbf{exploration}, particularly in multi-turn settings where compounding errors can reduce coverage.

Both our theoretical and empirical findings show that ALP significantly improves robustness and final performance. In summary, our contributions are three-fold:
\begin{enumerate}
    \item \textbf{General off-policy correction.} We propose ALP, which unifies staleness of training policies and training--inference mismatch into a single ratio of the updated and inference policies. ALP is simple and efficient to implement and avoids the over-truncation and multi-threshold tuning required by prior two-ratio approaches.
    \item \textbf{Theory: smoother geometry and bounded discrepancy.} We prove that with adaptive layerwise perturbation, KL divergence between the updated and inference distribution is bounded when \( \sigma^2 \) matches or exceeds the norm of the bias for the inference distribution from the training engine. This increases the probability that policy updates stay within a trust region and yields stable improvement. We further connect perturbation to loss smoothing, mitigating attraction to sharp, brittle optima.
    \item \textbf{Experiments: stability, performance, and ablations.} Across single-turn and multi-turn agentic reasoning tasks, ALP consistently outperforms MIS and Bypass by keeping more stable KL divergence and avoiding early plateau. We further show that ALP reduces off-policy mismatch and exhibits more stable optimization dynamics, consistent with improved exploration and robustness. Ablations further reveal that perturbing \emph{all layers} is most effective, and in partial settings, perturbations closer to lower layers tend to perform better than those restricted to upper layers.
\end{enumerate}

\subsection{Related Works}\label{s:Related Works}

\paragraph{Perturbation} Perturbation-based training is often motivated as a form of \emph{local averaging}: optimizing the expected objective under small input/latent disturbances can reduce sensitivity to sharp, brittle regions of the loss landscape and improve robustness under distribution shift. This perspective has been used across a range of learning settings to stabilize optimization and enhance generalization. A line of work has focused on enhancing model smoothness via perturbation-based sampling across diverse settings. For instance, \citet{moreno2018forward, yu2023noisynn} inject Gaussian noise to augment the training data and improve generalization. Certified robustness methods \citep{cohen2019rs, salman2019provably, lecuyer2019certified, yang2020randomized} pursue stability against adversarial attacks by introducing perturbations drawn from specific distributions during training. Beyond robustness, \citet{li2022positive, pereira2021multi} demonstrate that perturbation injection can also improve performance across a range of tasks. 
Another related line of research focuses on diffusion models \citep{ho2020denoising, song2020denoising, dhariwal2021diffusion, saharia2022palette, rombach2022high}, which model sample distributions as convolutions of Gaussian distributions and introduce Gaussian noise to the data during training.
More recently, \citet{shen2023engression, hao2025towards} propose learnable perturbation mechanisms that adaptively enhance model performance. Related ideas also appear in diffusion models, where additional perturbations can mitigate train--test mismatch in iterative generation: \citet{ning2023elucidating} reduce exposure bias by perturbing training inputs to better match the distribution encountered during sampling, while \citet{li2023alleviating} alleviate exposure bias via sampling with shifted time steps.

\section{\alg}

\subsection{Prior Approaches}\label{sec:prior_approaches}
For prompts $x\sim d_0$, the response $a$ is generated from the rollout policy $\pi_{\theta_\old}^\infer(\cdot|x)$, where $\theta_\old$ is the parameter from the last step. We then update $\theta$ using training probability $\pi_\theta^\train$. Then, the standard GRPO
\begin{equation}
\begin{aligned}
    \mathcal{J}_{\mathrm{GRPO}}(\theta) =& \mathbb{E}_{x \sim d_0, \{a_i\}_{i=1}^n\sim\pi_{\theta_{\old}}^{\infer}(\cdot|x)} \bigg[ \frac{1}{n} \sum_{i=1}^{n}\sum_{t=1}^{|a_i|}\min \Big\{ \rho_{i,t}(\theta) A_i, \text{clip} \left( \rho_{i,t}(\theta), 1-\epsilon, 1+\epsilon \right) A_i \Big\} \bigg],
\end{aligned}
\end{equation}
where $|a_i|$ denotes the length of $a_i$, $A_i$ is the group advantage, and $\rho_{i,t}(\theta)$ is the token-level importance ratio at the $t$-th token:
$
\rho_{i,t}(\theta)= \pi^{\train}_{\theta}(a_{i,t}|x, a_{i,<t})/\pi^{\train}_{\theta_\old}(a_{i,t}|x, a_{i,<t}).
$

\paragraph{Sequence-level objective.}
However, this token-level objective is biased when the difference between $\pi^{\train}_{\theta}$ and $\pi^{\train}_{\theta_\old}$ cannot be neglected. We can use the unbiased sequence level objective \citep{zheng2025group}, 
\[
\rho_i(\theta)=\frac{\pi^{\train}_{\theta}(a_i|x)}{\pi^{\train}_{\theta_\old}(a_i|x)} = \prod_{t=1}^{|a_i|} \frac{\pi^{\train}_{\theta}(a_{i,t}|x, a_{i,<t})}{\pi^{\train}_{\theta_\old}(a_{i,t}|x, a_{i,<t})}.
\]
To balance the bias and variance, Group Sequence Policy Optimization (GSPO) \citep{zheng2025group} proposes using the geometric mean:
$
(\pi^{\train}_{\theta}(a_i|x)/\pi^{\train}_{\theta_\old}(a_i|x))^{1/|a_i|}.
$

\paragraph{Masked importance sampling (MIS).} Since $\pi_{\theta_{\old}}^{\infer}$ and $\pi^{\train}_{\theta_\old}$ have a mismatch, to make algorithms robust to this mismatch, a standard approach is to multiply by a correcting ratio \citep{yao2025rollouttrainingmismatch,liu2025flashrl}:
\begin{equation*}
\begin{aligned}
    \mathcal{J}_{\mathrm{MIS}}(\theta) =& \mathbb{E}_{x \sim d_0, \{a_i\}_{i=1}^n\sim\pi_{\theta_{\old}}^{\infer}(\cdot|x)} \bigg[ \frac{1}{n} \sum_{i=1}^{n}\sum_{t=1}^{|a_i|}\mathrm{Mask}\left\{\rho'_{i,t}, C\right\} \cdot \min \Big\{ \rho_{i,t}(\theta) A_i, \text{clip} \left( \rho_{i,t}(\theta), 1-\epsilon, 1+\epsilon \right) A_i \Big\} \bigg],
\end{aligned}
\end{equation*}
where $C$ is a threshold parameter, the gradient is masked when $\rho'_{i,t}>C$, and the additional importance ratio 
$
\rho'_{i,t} = \pi^{\train}_{\theta_\old}(a_{i,t}|x, a_{i,<t})/\pi^{\infer}_{\theta_\old}(a_{i,t}|x, a_{i,<t}).
$
This ratio can also be calculated in sequence level $\rho'_{i}=\prod_t\rho'_{i,t}$.

\paragraph{Bypass fails to stabilize training.}
A direct way to unify the importance ratio is to replace $\rho_i(\theta)$ with 
$
\pi^{\train}_{\theta}(a_{i,t}|x, a_{i,<t})/\pi^{\infer}_{\theta_\old}(a_{i,t}|x, a_{i,<t}).
$
This approach is a baseline in training--inference mismatch work \citep{yao2025rollouttrainingmismatch,li_rl_collapse}, but it does not work effectively. Specifically, since the family class of $\pi_{\theta_\old}^\infer$ has system bias relative to the training policy $\pi_{\theta_\old}$, the policy gradients contain noise and bias, making the policy more unstable and brittle. Empirically, Bypass exhibits blow-up in the ratio tail during full training in Figure~\ref{fig:alp} (right). Even with the same checkpoint and the same rollouts, the (log-)ratio envelope can flare up in the low-probability tail without additional mechanisms to control mismatch (Fig.~\ref{fig:smooth}, middle).

\subsection{Our Approach}
For layer $h\in[H]$ in the model, suppose the dimension of its input embedding is $d_h$. We add a Gaussian
perturbation variable $\delta^h\sim\mathcal{N}(0, \sigma_h^2 I_{d_h})$ to the updated training policy. Denoting the perturbation variable $\delta=[\delta^1,\ldots,\delta^H]^T\sim\mathcal{N}(0, \mathrm{diag}(\sigma^2))$, where $\sigma$ is the std vector and $d=\sum_{h=1}^Hd_h$, we define the perturbed policy
$
\pi_{\theta,\sigma}(a|x,\delta) = \prod_{t=1}^{|a|}\pi_{\theta,\sigma}(a_t|x, a_{<t},\delta_t),
$
where $\delta_t$ matches the hidden-state tensor shape and is sampled independently across token positions.
Then, the loss function is
\begin{equation}
\footnotesize
\begin{aligned}\label{eq:token-alp}
    \mathcal{J}_{\mathrm{ALP}}(\theta,\sigma) =& \mathbb{E}_{x \sim d_0, \{a_i\}_{i=1}^n\sim\pi_{\theta_{\old}}^{\infer}(\cdot|x), \delta_i\sim\mathcal{N}(0, \mathrm{diag}(\sigma^2))} \bigg[ \frac{1}{n} \sum_{i=1}^{n}\sum_{t=1}^{|a_i|}\min \Big\{ \rho^{\mathrm{ALP}}_{i,t}(\theta) A_i, \text{clip} \left( \rho^{\mathrm{ALP}}_{i,t}(\theta), 1-\epsilon_l, 1+\epsilon_h \right) A_i \Big\} \bigg],
\end{aligned}
\end{equation}
where the token-level importance ratio
\begin{equation*}
\begin{aligned}
    \rho^{\mathrm{ALP}}_{i,t} = \frac{\pi_{\theta,\sigma}(a_{i,t}|x,a_{i,<t},\delta_{i,t})}{\pi^{\infer}_{\theta_\old}(a_{i,t}|x,a_{i,<t})},
\end{aligned}
\end{equation*}
and the sequence-level importance ratio
\begin{equation}\label{eq:seq-alp}
\rho^{\mathrm{ALP}}_{i} = \prod_{t=1}^{|a_i|} \rho^{\mathrm{ALP}}_{i,t} = \frac{\pi_{\theta,\sigma}(a_i|x,\delta_i)}{\pi^{\infer}_{\theta_\old}(a_i|x)}
\end{equation}
can also be applied to the loss correspondingly. Note that only the training policy is perturbed not the inference one.

\subsection{Analysis: Robustness of ALP}
\begin{figure}[t]
    \centering
    \begin{subfigure}[t]{0.38\textwidth}
        \centering
        \includegraphics[width=\textwidth]{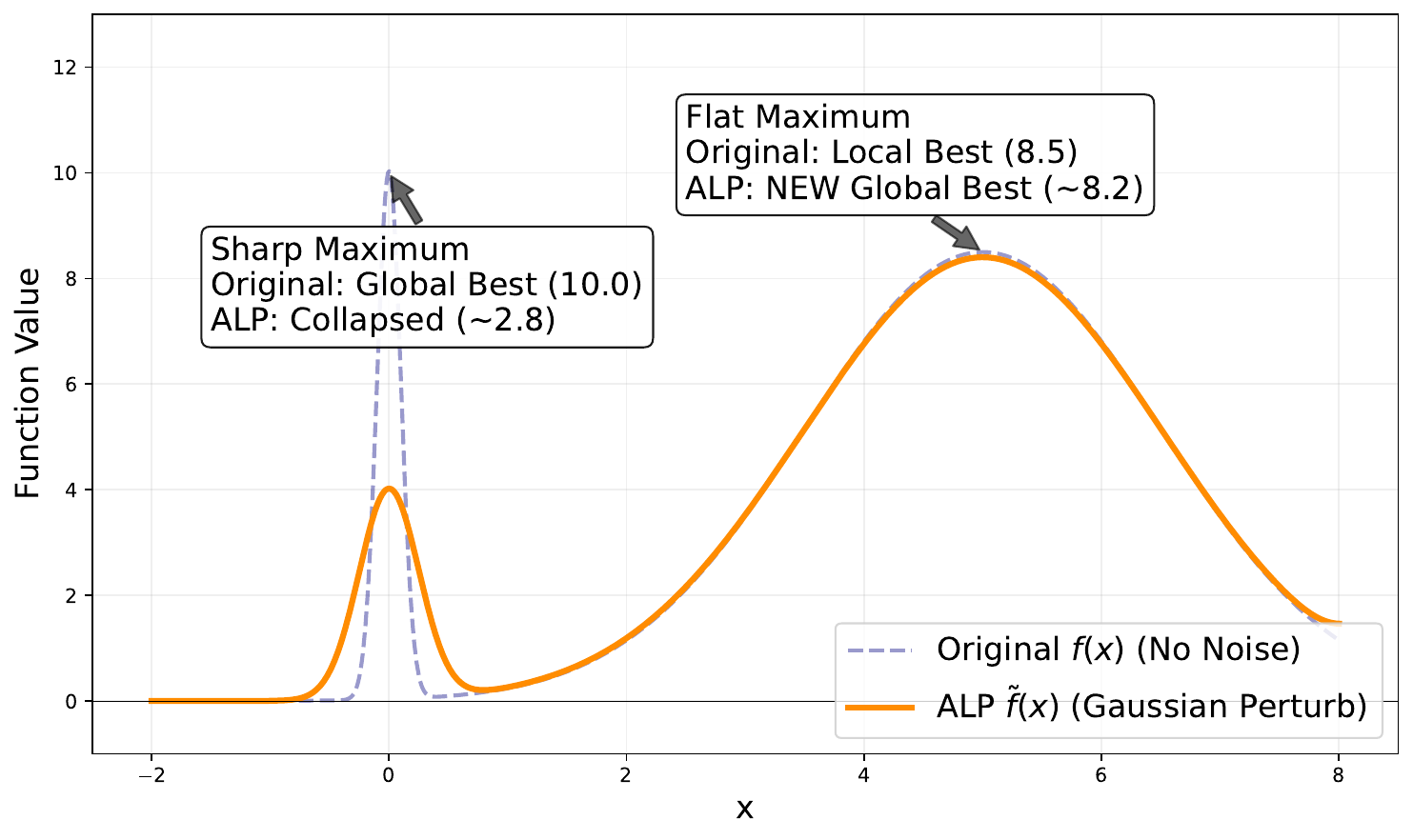}
    \end{subfigure}
    \hfill
    \begin{subfigure}[t]{0.3\textwidth}
        \centering
        \includegraphics[width=\textwidth]{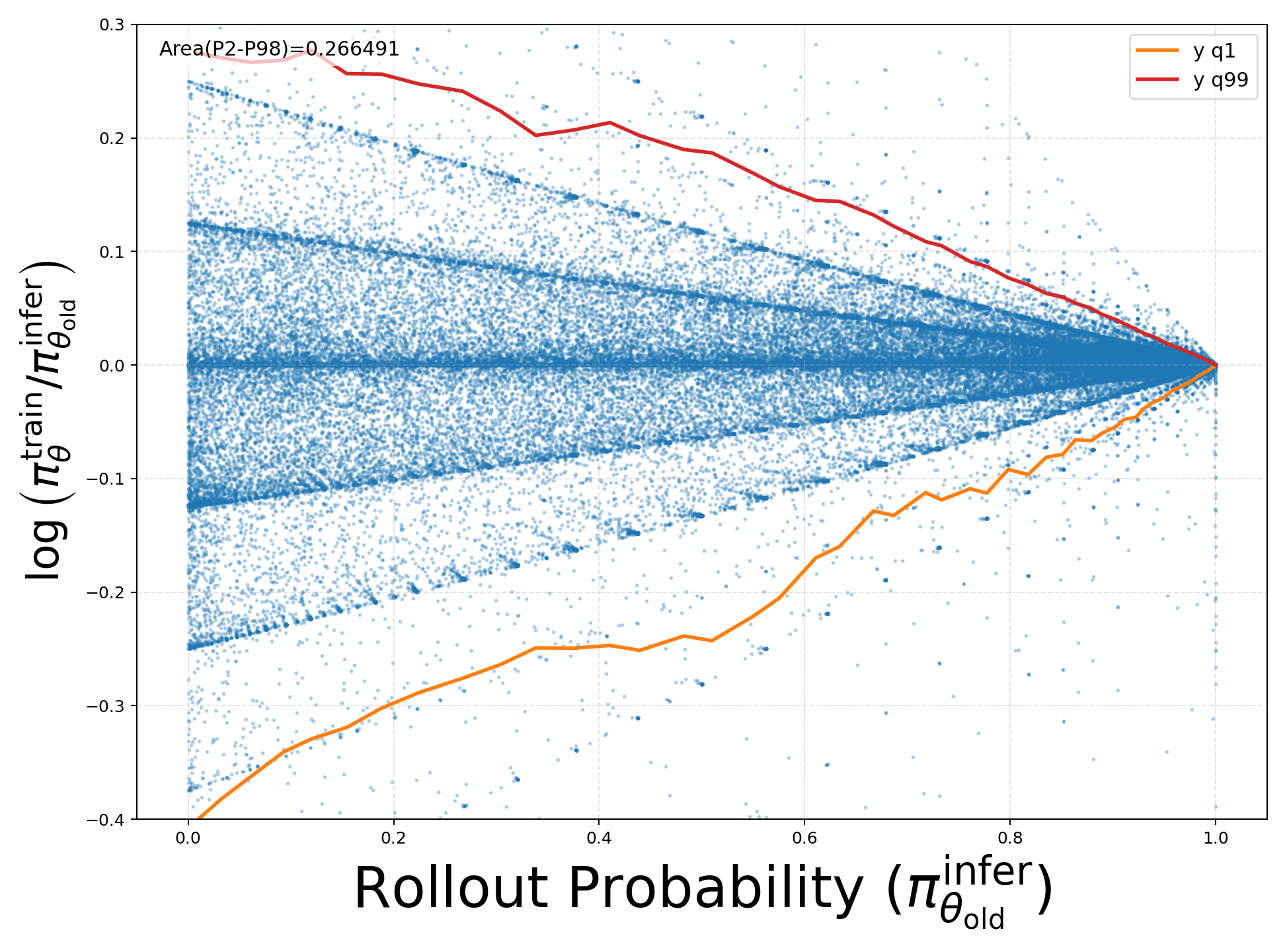}
    \end{subfigure}
    \hfill
    \begin{subfigure}[t]{0.3\textwidth}
        \centering
        \includegraphics[width=\textwidth]{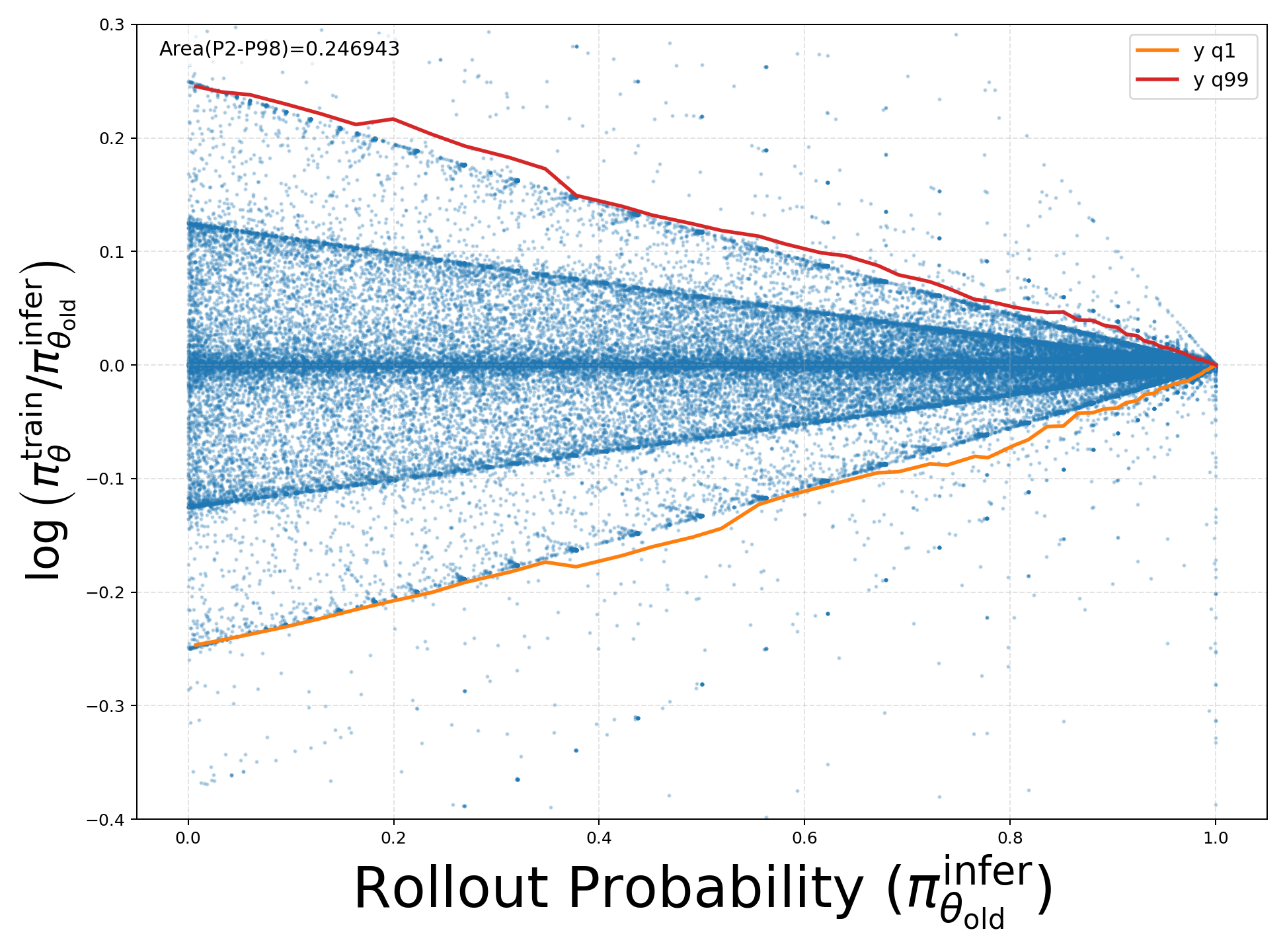}
    \end{subfigure}
    \hfill
    \caption{Effects of perturbation. Left: perturbation smooths a sharp objective into a flatter surrogate. Middle/right: in a controlled multi-turn comparison from the same checkpoint and rollout batch, perturbation shrinks the mismatch envelope, especially in the low-probability tail.}
    \label{fig:smooth}
\end{figure}

ALP helps in two ways: it reduces training-inference mismatch and smooths the local optimization geometry. For clarity, we analyze a one-layer perturbed policy
\[
    \tilde{\pi}_\theta(a|x) = \E_\delta \pi_\theta(a|x+\delta),
\]
where the expectation is used only for analysis; the algorithm uses a single-sample approximation in practice. Formal statements and proofs are deferred to Appendix~\ref{pf:mismatch} and Appendix~\ref{pf:smooth}.

\paragraph{Controlling training-inference mismatch.}
We model the inference-time policy as the training-time policy with small noises on the input of each layer:
$
    \pi_{\theta_\old}^\infer(a|x) \approx \pi_{\theta_\old}(a|x+\zeta),
$
where $\zeta$ is a zero-mean system-induced bias. Without perturbation, even a small $\zeta$ can be amplified by a sharp local policy geometry, leading to unstable importance ratios. ALP mitigates this effect by enlarging the effective policy family around the training distribution and covering mismatch noise with perturbation.
\begin{theorem}[Informal]\label{thm:mismatch_informal}
When perturbation is not too large to distort the original distribution, we have
\begingroup\small
\begin{equation}\label{eq:kl_mismatch}
\begin{aligned}
   \E_{x \sim d_0,\zeta} \mathrm{KL} \left( \tilde{\pi}_{\theta_\old} (a |x) \parallel \pi_{\theta_\old}^\infer(a|x)  \right) \le  \mathcal{O}\left(\frac{d \E \| \zeta \|_2^2}{2 \sigma^2}\right),
\end{aligned}
\end{equation}
\endgroup
where $\mathcal{O}(\cdot)$ hides absolute constants and lower-order terms.
\end{theorem}
The formal one is deferred to Theorem \ref{thm:mismatch}.
The theorem shows the key trade-off: $\sigma$ should be large enough to cover mismatch noise but small enough to preserve the original policy. In turn, the updated policy is more likely to stay close to the rollout distribution and remain within a trust region \citep{schulman2015trust}.

\paragraph{Improvement on Loss Smoothness.}
Beyond mismatch correction, perturbation acts as local averaging: instead of optimizing a sharp objective at a single representation, ALP optimizes an averaged neighborhood around it. This suppresses brittle local curvature and reduces sensitivity to spiky optima.
\begin{theorem}[Informal]\label{thm:smooth_informal}
Under mild regularity conditions, perturbation reduces the effective curvature of the surrogate objective, yielding a smoother optimization landscape than the non-perturbed objective.
\end{theorem}
We provide the formal theorem in Theorem \ref{thm:smooth}.
This result formalizes the intuition that ALP discourages optimization from collapsing onto sharp, fragile regions. Figure~\ref{fig:smooth} provides empirical support: the toy example illustrates smoothing, while the controlled comparison shows a visibly tighter log-ratio envelope after applying perturbation. Additionally, in a controlled one-step intervention from the same checkpoint (Figure~\ref{fig:smooth}, middle/right), adding perturbation substantially narrows the conditional quantile envelope of the log-ratio,
especially for low-probability tokens that dominate tail risk. This indicates that ALP also suppresses extreme training-inference deviations in a single update step, increasing the likelihood that updates remain within a trust region around the rollout distribution \citep{schulman2015trust}.

\section{Experiments}

\subsection{Single-turn Reasoning}\label{sec:single_turn_setting}

\subsubsection{Experimental settings}

\textbf{Dataset.} For \emph{single-turn} reasoning, we study math tasks without tool use. The training set merges the math subset of Guru RL-92k\footnote{\url{https://huggingface.co/datasets/LLM360/guru-RL-92k}}\citep{cheng2025revisiting} and a $75$K OpenR1\footnote{\url{https://huggingface.co/datasets/weqweasdas/from_default_filtered_openr1_with_scores}}\citep{openr1} subset, with correctness verified by \textbf{Math-Verify} \citep{Kydlicek_Math-Verify_Math_Verification}. Additional preprocessing details are deferred to Appendix~\ref{sec:Additional Experimental Details}.

\textbf{Training.}
We implement all methods in verl \citep{sheng2024hybridflow}, following its default setup where possible. We use AdamW with learning rate $1\times 10^{-6}$ and the clip-higher trick \citep{yu2025dapoopensourcellmreinforcement} with asymmetric clipping $(1-\epsilon_l, 1+\epsilon_h)$. The models are trained from Qwen2.5-Math-1.5B-base \citep{yang2024qwen2}; each iteration samples $512$ prompts, rolls out $n=8$ responses per prompt at temperature $1.0$, and performs $16$ policy updates. As provided in Table~\ref{t:Settings of parameters of importance ratios} in the appendix, token/seq denotes whether the ratio is aggregated at the token level or sequence level. The definitions of these notations are provided in Sec.~\ref{sec:prior_approaches}. Unless otherwise specified, the main single-turn experiments use Qwen2.5-Math-1.5B-base; we additionally report a validation run on Qwen3-4B \citep{yang2025qwen3} to test whether observed trend transfers to a different backbone.

\textbf{Benchmarks.}
We evaluate on Math500 \citep{hendrycks2021measuring}, Minerva Math \citep{lewkowycz2022solving}, Olympiad Bench \citep{he2024olympiadbench}, AIME2024\footnote{https://huggingface.co/datasets/math-ai/aime24}, and AIME2025\footnote{https://huggingface.co/datasets/math-ai/aime25}. We report average@32 at temperature $1.0$ with a generation limit of $4096$ tokens. Baselines include Seq-Bypass, Seq-MIS, Token-MIS, and vanilla GRPO; Table~\ref{t:Settings of parameters of importance ratios} summarizes how each method forms and clips the importance ratios.

\subsubsection{Main Results}
\begin{table}[ht]
\small
\centering
\setlength{\tabcolsep}{6pt}
\begin{tabular}{lcccccc}
\toprule
Method & Math500 & Minerva Math & Olympiad Bench & AIME24 & AIME25 & Average \\
\midrule
GRPO       & 75.91 & 36.43 & 38.82 & 16.77 & 10.94 & 35.77 \\
Seq-Bypass & 76.21 & 35.23 & 38.52 & 16.35 & 7.81 & 34.82 \\
Seq-MIS    & 77.08 & 36.75 & 39.06 & 15.21 & 9.58 & 35.54 \\
Token-MIS  & 77.84 & 35.94 & 40.06 & 17.40 & 10.83 & 36.41 \\
\hline
\textbf{Token-ALP}  & \textbf{78.10} & \textbf{37.27} & \textbf{40.77} & \textbf{21.46} & 11.77
& \textbf{37.87} \\
Seq-ALP    & 77.84 & 37.06 & 40.28 & 16.98 & \textbf{11.98} & 36.83 \\
\bottomrule
\end{tabular}
\caption{Single-turn evaluation (average@32, temperature $1.0$) across five math benchmarks. Token-ALP attains the highest average score.}
\label{tab:single_turn_results}
\end{table}

\begin{figure}[ht]
    \centering
    \includegraphics[width=0.95\linewidth]{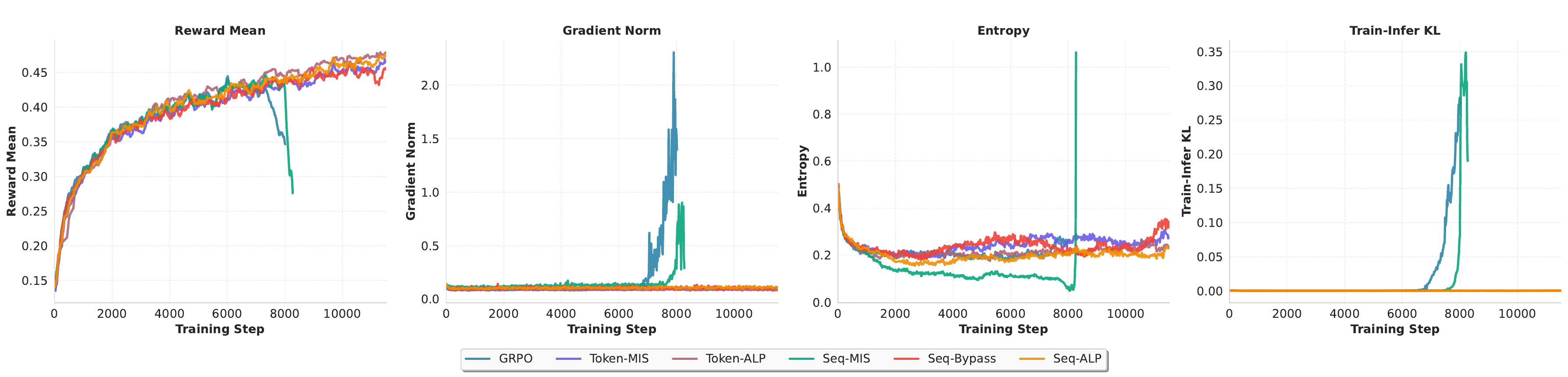}
    \caption{Single-turn training dynamics under training-inference mismatch: reward, gradient norm, entropy, and rollout-training KL over policy-update steps. Reward is smoothed with a 10-step moving average.}
    \label{fig:single_turn_comparison}
\end{figure}

\begin{table}[ht]
\small
\centering
\setlength{\tabcolsep}{6pt}
\begin{tabular}{lcccccc}
\toprule
Method & Math500 & Minerva Math & Olympiad Bench & AIME24 & AIME25 & Average \\
\midrule
GRPO       & 94.49 & 56.96 & 68.28 & 60.29 & 51.17 & 66.24 \\
Token-MIS  & 94.28 & 53.66 & 66.36 & 62.35 & 50.00 & 65.33 \\
Seq-Bypass & 93.70 & 52.00 & 64.21 & 63.80 & 46.88 & 64.12 \\
\hline
Seq-ALP    & \textbf{95.18} & \textbf{57.44} & \textbf{69.49} & \textbf{64.32} & \textbf{51.17} & \textbf{67.52} \\
\bottomrule
\end{tabular}
\caption{Single-turn validation on Qwen3-4B under the same evaluation protocol. Seq-ALP achieves the best average score across benchmarks.}
\label{tab:single_turn_qwen3_4b}
\end{table}

\textbf{ALP-type delivers the top performance.}
Under single-turn evaluation in Table \ref{tab:single_turn_results}, Token-ALP attains the best overall average score ($37.87$), with Seq-ALP second ($36.83$). Token-MIS is the strongest baseline, while Seq-MIS, GRPO, and Seq-Bypass perform worse overall. To test whether the benefit of perturbation is specific to Qwen2.5-Math-1.5B, we further repeat the single-turn experiment on Qwen3-4B. Table~\ref{tab:single_turn_qwen3_4b} shows that Seq-ALP achieves the best overall average score ($67.52$), outperforming GRPO ($66.24$), Token-MIS ($65.33$), and Seq-Bypass ($64.12$). This suggests that the gain from ALP is not tied to one particular base model and persists at a larger model scale.


\textbf{ALP consistently improves training stability.}
Figure \ref{fig:single_turn_comparison} shows that both token-level and sequence-level ALP remain stable throughout training: gradient norms stay controlled, entropy remains in a reasonable range, and rollout-training KL stays near zero. By contrast, GRPO and Seq-MIS exhibit late-stage spikes in gradient norm and KL together with reward collapse, while Seq-Bypass and Token-MIS are more stable but still less competitive than ALP overall.

\section{Multi-Turn Agentic Reasoning}

\subsection{Experimental settings}

We study multi-turn Tool-Integrated Reasoning (TIR) with a Python interpreter. We adapt the SimpleTIR codebase \citep{xue2025simpletir}, mask void turns in the loss, and train on data merged from SimpleRL \citep{zeng2025simplerl}, Deepscaler \citep{deepscaler2025}, and rStar2-Agent \citep{shang2025rstar2}. ALP follows the zero-RL setting and starts from Qwen2.5-7B-base. The rollout batch size is $512$, the mini-update size is $32$, the maximum response length is $16$K, and each episode allows up to $5$ code-execution turns. We use the same baselines and evaluation benchmarks as in the single-turn setting.

\subsection{Main Results}

\begin{table}[ht]
\small
\centering
\setlength{\tabcolsep}{6pt}
\begin{tabular}{c|cccccc}
\hline
Algorithm & Math500 & Minerva Math & Olympiad Bench & AIME24 & AIME25 & Average \\
\hline
GRPO                & 80.90 & 42.13 & 49.59 & 35.42 & 24.79 & 46.57\\
Seq-Bypass                & 81.74 & 39.21 & 51.08 & 37.19 & 24.06 & 46.66\\
Token-MIS                & 83.29 & 41.45 & 50.65 & 39.48 & 28.85 & 48.74\\
Seq-MIS                & 80.61 & 42.05 & 48.38 & 39.48 & 24.17 & 46.94\\
\hline
Token-ALP                & 83.48 & \textbf{43.18} & 51.55 & 38.65 & \textbf{31.25} & 49.62\\
\textbf{Seq-ALP}               & \textbf{84.29} & 43.10 & \textbf{52.75 }& \textbf{43.85} & 28.65 & \textbf{50.53}\\
\hline
\end{tabular}
\caption{TIR evaluation: the test accuracy average@32 under temperature $1.0$ across the five benchmarks. ALP has the best performance.}
\label{tab:multi-turn perform}
\end{table}

\begin{figure}[ht]
    \centering
    \begin{subfigure}[t]{0.32\textwidth}
        \centering
        \includegraphics[width=\textwidth]{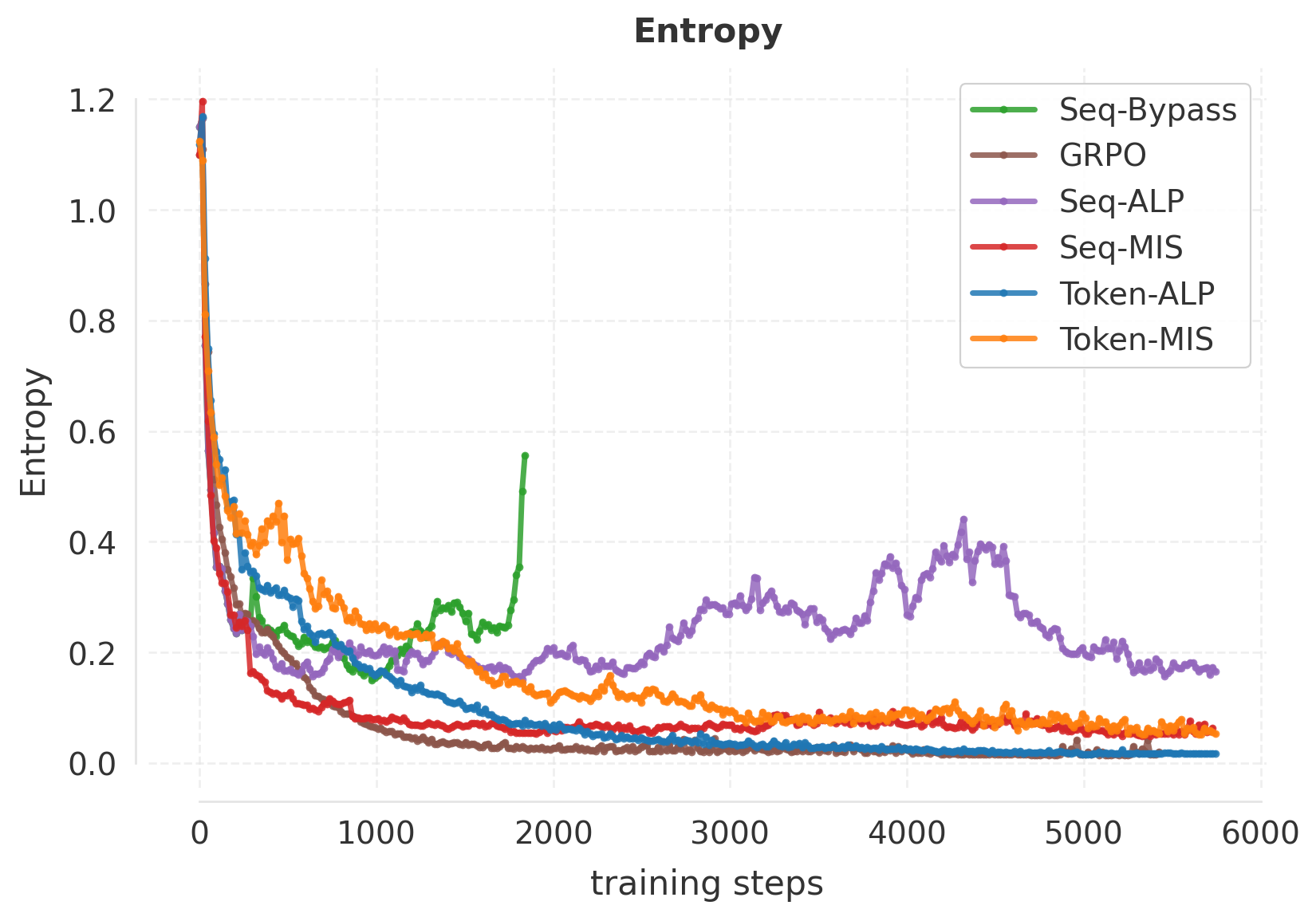}
    \end{subfigure}
    \hfill
    \begin{subfigure}[t]{0.32\textwidth}
        \centering
        \includegraphics[width=\textwidth]{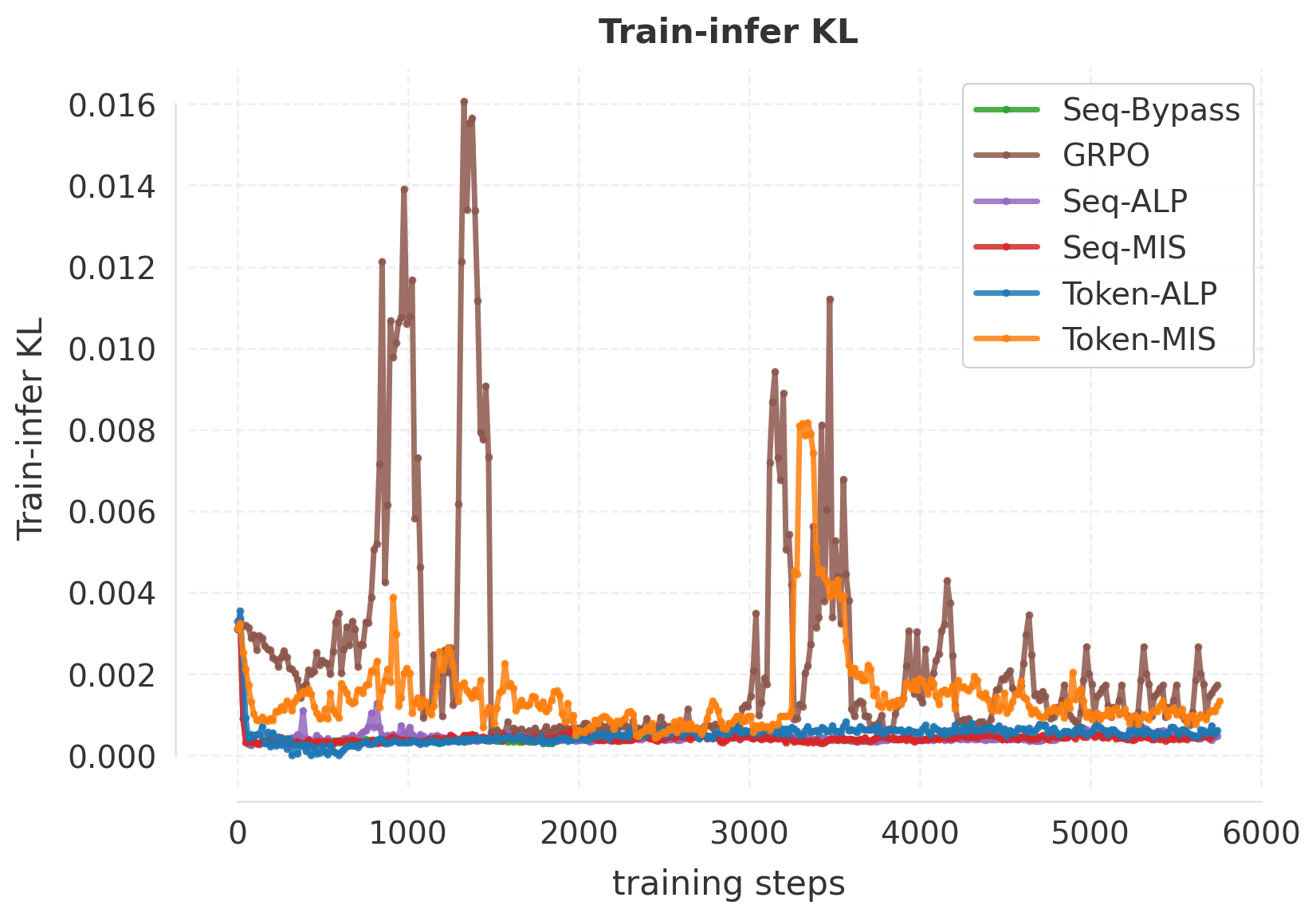}
    \end{subfigure}
    \begin{subfigure}[t]{0.32\textwidth}
        \centering
        \includegraphics[width=\textwidth]{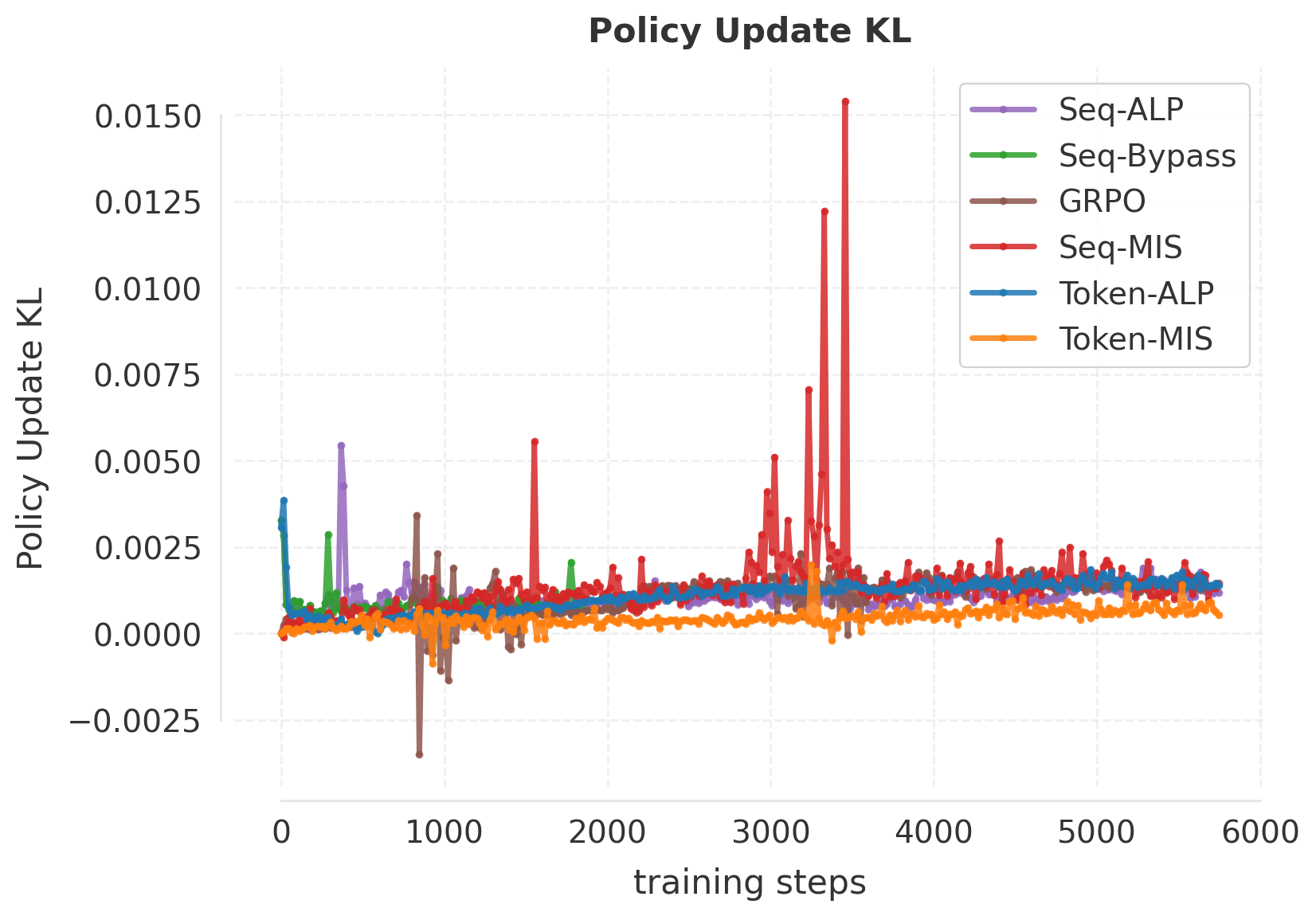}
    \end{subfigure}
    \caption{Multi-turn training dynamics. Seq-ALP maintains higher but controlled entropy and stable convergence in both KL metrics, while Token-MIS is unstable in train-inference KL and Seq-MIS shows spikes in policy-update KL.}
    \label{fig:metric_tir}
\end{figure}

\begin{figure}[ht]
    \centering
    \begin{subfigure}[t]{0.45\textwidth}
        \centering
        \includegraphics[width=\textwidth]{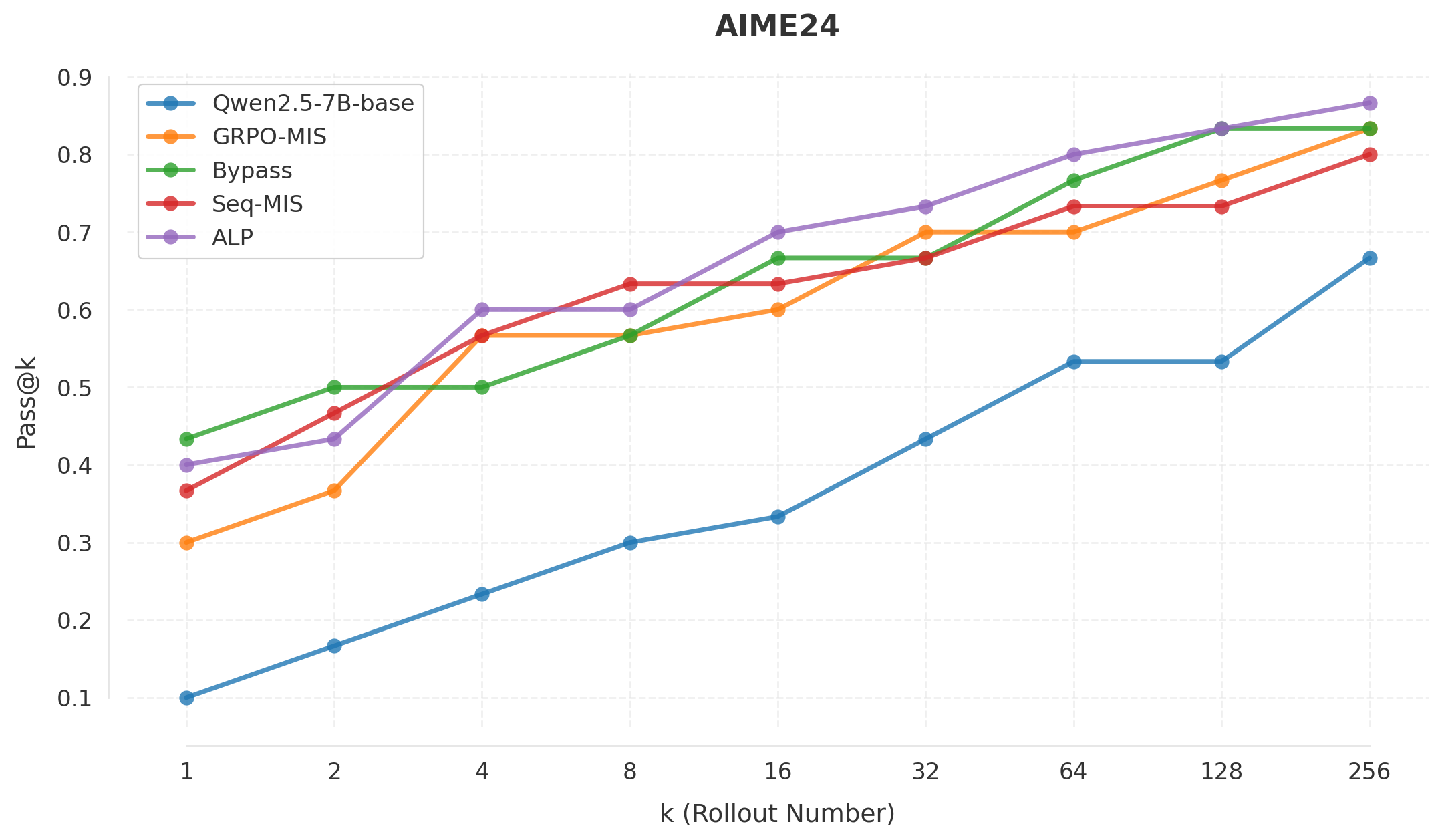}
    \end{subfigure}
    \hfill
    \begin{subfigure}[t]{0.45\textwidth}
        \centering
        \includegraphics[width=\textwidth]{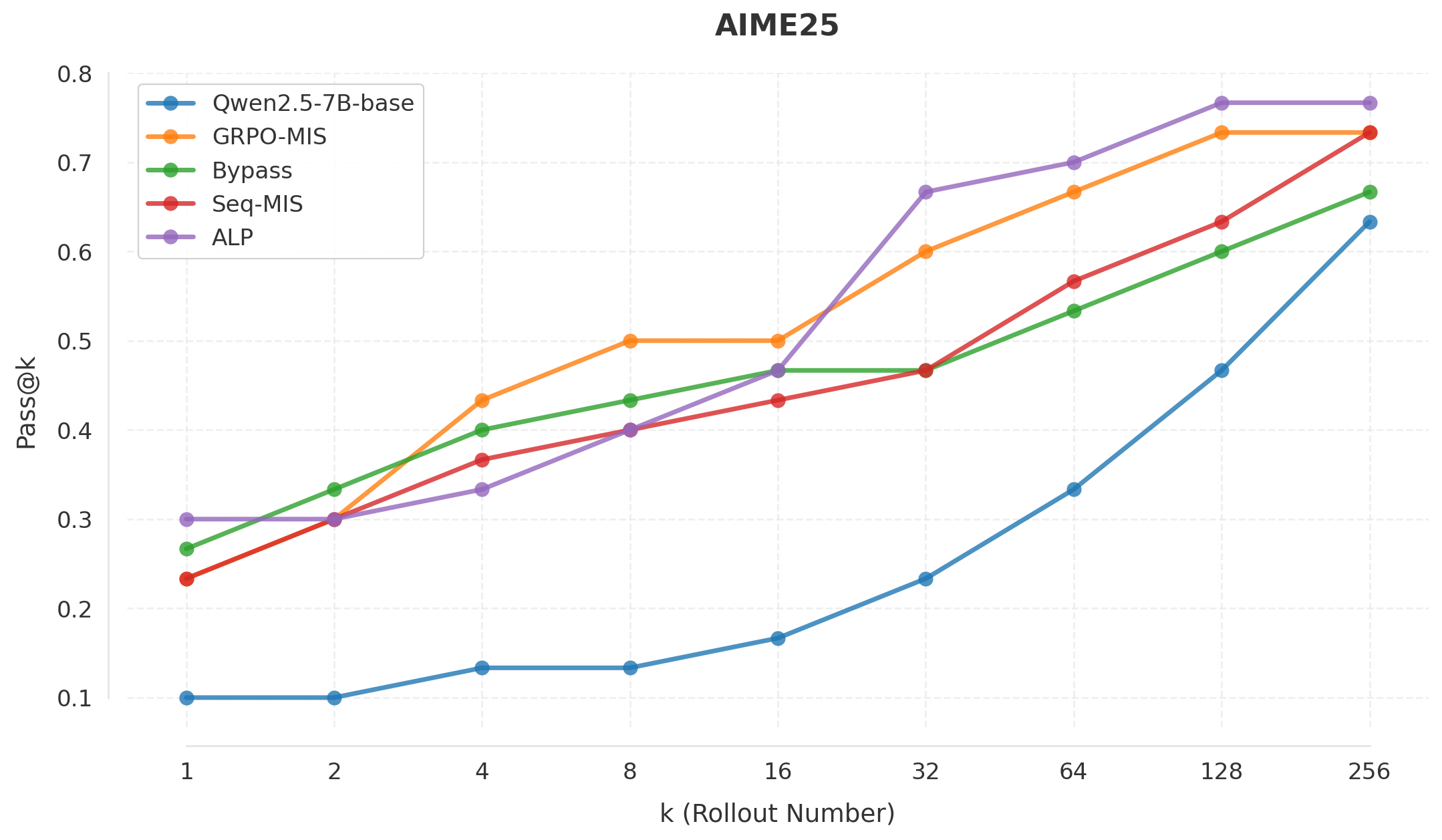}
    \end{subfigure}
    \caption{Pass@$k$ on AIME 2024 and AIME 2025 for TIR. ALP achieves the highest accuracy over rollout budgets $k=16\sim256$, indicating improved exploration efficiency and solution diversity.}
    \label{fig:passk}
\end{figure}

\textbf{ALP has the best performance.} On multi-turn TIR benchmarks in Table \ref{tab:multi-turn perform}, Seq-ALP achieves the highest overall average accuracy ($50.53$), outperforming Token-MIS, Seq-MIS, and Seq-Bypass. The gains are broad-based across benchmarks, with only AIME25 slightly favoring Token-MIS.

\textbf{ALP enhances training stability.} Figure \ref{fig:metric_tir} shows that Seq-ALP maintains controlled train-inference KL and preserves higher but stable entropy. While the reward for MIS does not collapse, their KL divergence still collapse and the additionally masking easily results in early plateau.

\begin{remark}[Token level and sequence level]
    An observation is that token-level ALP performs best in single-turn settings, while sequence-level ALP performs best in multi-turn settings. A plausible explanation is an efficiency–stability trade-off. Token-level ratios can be biased under policy drift, but in relatively stable single-turn training this bias is small, and clipping only affects the few tokens whose ratios exceed the threshold, thus preserving most learning signal. In contrast, sequence-level ratios are less biased in principle, but are more sensitive to outliers: extreme tokens can cause the entire trajectory to be clipped, discarding substantial amount of signal. As drift and mismatch become more severe in multi-turn training, the token-level bias can grow, making the lower-bias sequence-level variant more stable.
\end{remark}

\textbf{Perturbation promotes exploration efficiency.} Finally, the Pass@$k$ curves on AIME24 and 25 provide direct evidence that latent-space perturbations translate into more effective exploration (Figure~\ref{fig:passk}). Across moderate-to-large rollout budgets $k$ ($16$--$256$), ALP consistently attains the highest Pass@$k$, with the advantage becoming more pronounced as $k$ increases, suggesting that perturbing hidden states helps generate a more diverse set of candidate solution trajectories and thus improves the probability of sampling a correct reasoning path under fixed decoding temperature.

\section{Ablations}

\subsection{Perturbation targets}\label{sec:ablation_targets}
\begin{figure}[t]
    \centering
    \includegraphics[width=0.8\linewidth]{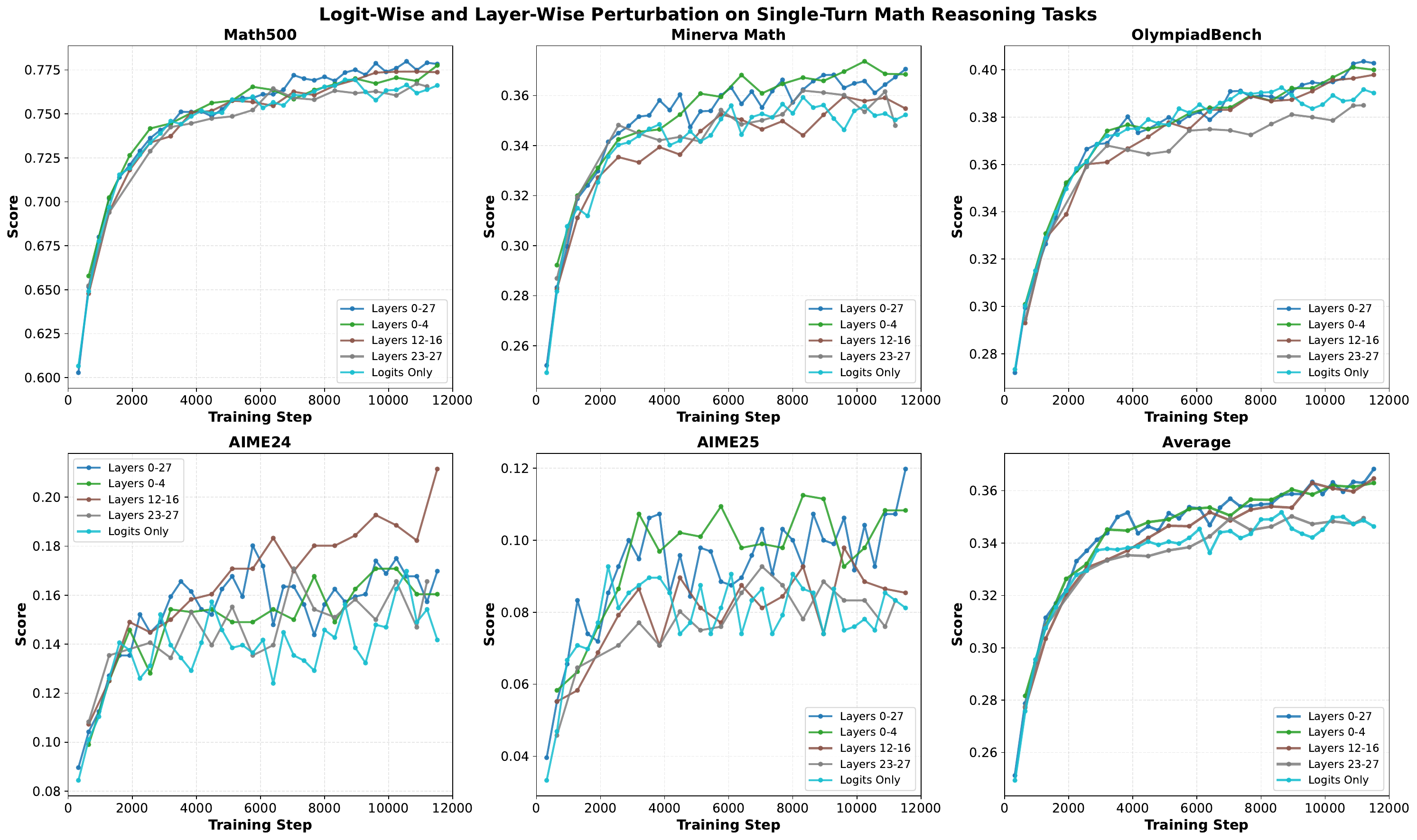}
    \caption{Perturbation-layer ablation in the single-turn setting. The results reveal a clear performance hierarchy: All layers $>$ Partial layers $>$ Logits only; Early layers $>$ Late layers.}
    \label{fig:ablation_targets}
\end{figure}

We study how the \emph{location} of Gaussian perturbations affects both performance and stability. Starting from the same base setup, we inject perturbations into (i) all transformer layers (0--27), (ii) early layers (0--4), (iii) middle layers (12--16), (iv) late layers (19--23), or (v) the output logits only (no hidden-state perturbation).

\paragraph{Main finding: All layers $>$ partial layers $>$ logits-only.}
Across single-turn benchmarks (Figure~\ref{fig:ablation_targets}), perturbing \emph{all} layers consistently achieves the strongest average accuracy, while restricting perturbations to a narrow band of layers is weaker. Logits-only perturbation is the least effective. This pattern is consistent with the view that ALP benefits from representation-level family enlargement rather than output noise alone, since logits-only perturbation is substantially weaker.

\paragraph{Multi-turn Ablation.}
\begin{table}[t]
\small
\centering
\setlength{\tabcolsep}{6pt}
\begin{tabular}{c|cccccc}
\hline
Algorithm & Math500 & Minerva Math & Olympiad Bench & AIME24 & AIME25 & Average \\
\hline
Layers 0-5                      & 83.58 & 43.62 & 49.88 & 38.85 & 25.31 & 48.25\\
Layers 12-17                    & 82.19 & 41.70 & 51.69 & 39.69 & 27.29 & 48.51\\
Layers 23-27                    & 82.77 & 43.97 & 49.60 & 40.10 & 26.87 & 48.66\\
\textbf{Layers 0-27}               & \textbf{84.29} & \textbf{43.10} & \textbf{52.75 }& \textbf{43.85} & 28.65 & \textbf{50.53}\\
\hline
\end{tabular}
\caption{Perturbation-layer ablation in the multi-turn setting. All-layer perturbation outperforms partial-layer perturbations.}
\label{tab:multi-turn layer ablat}
\end{table}

The same conclusion holds in multi-turn agentic evaluation (Table~\ref{tab:multi-turn layer ablat}). All-layer perturbation attains the best overall score across benchmarks, whereas partial-layer perturbations trail behind. Since multi-turn rollouts compound distributional drift over time, this result indicates that perturbing broadly across depth is particularly effective at maintaining robustness under sequential feedback and tool-induced distribution shift.

\paragraph{Why do partial or late-layer perturbations under-perform?}
Layer ablations suggest that perturbing all layers is the most reliable choice across settings; in single-turn tasks, earlier-layer perturbations can outperform later-layer ones. We interpret this through a \emph{distribution-family} lens: injecting perturbations into hidden states enlarges the effective policy family $\{\pi_{\theta,\sigma}\}$, providing additional degrees of freedom to absorb inference-time distortions and increase support overlap with the rollout distribution. In contrast, restricting perturbations to a narrow late-layer band expands the family in a more limited and output-local way, which is less effective at matching the rollout distribution and more likely to concentrate changes near the decision boundary. Consistent with this view, our diagnostics show that all-layer perturbation achieves the smallest tail deviations in probability ratios, whereas late-layer-only perturbation exhibits heavier tails.

\section{Conclusions}
We introduced \textbf{Adaptive Layerwise Perturbation (ALP)}, a unified approach for stabilizing off-policy RL in large language models. ALP injects learnable hidden-state perturbations when computing the updated policy and optimizes a single importance ratio against the rollout distribution, thereby addressing policy staleness and training-inference mismatch within one formulation. Our analysis suggests that perturbation helps both by tightening the train-infer discrepancy and by smoothing sharp local optimization geometry. Empirically, across single-turn and multi-turn reasoning tasks, ALP improves both training stability and final performance, while ablations show that hidden-state perturbation, especially across all layers, is most effective. Promising future directions include fully asynchronous RL and Mixture-of-Experts settings, where off-policy effects and systems mismatch can be even more severe.

\bibliography{main}

@article{ho2020denoising,
  title={Denoising diffusion probabilistic models},
  author={Ho, Jonathan and Jain, Ajay and Abbeel, Pieter},
  journal={Advances in neural information processing systems},
  volume={33},
  pages={6840--6851},
  year={2020}
}

@article{song2020denoising,
  title={Denoising diffusion implicit models},
  author={Song, Jiaming and Meng, Chenlin and Ermon, Stefano},
  journal={arXiv preprint arXiv:2010.02502},
  year={2020}
}

@article{dhariwal2021diffusion,
  title={Diffusion models beat gans on image synthesis},
  author={Dhariwal, Prafulla and Nichol, Alexander},
  journal={Advances in neural information processing systems},
  volume={34},
  pages={8780--8794},
  year={2021}
}

@inproceedings{saharia2022palette,
  title={Palette: Image-to-image diffusion models},
  author={Saharia, Chitwan and Chan, William and Chang, Huiwen and Lee, Chris and Ho, Jonathan and Salimans, Tim and Fleet, David and Norouzi, Mohammad},
  booktitle={ACM SIGGRAPH 2022 conference proceedings},
  pages={1--10},
  year={2022}
}

@inproceedings{rombach2022high,
  title={High-resolution image synthesis with latent diffusion models},
  author={Rombach, Robin and Blattmann, Andreas and Lorenz, Dominik and Esser, Patrick and Ommer, Bj{\"o}rn},
  booktitle={Proceedings of the IEEE/CVF conference on computer vision and pattern recognition},
  pages={10684--10695},
  year={2022}
}

@article{shen2023engression,
  title={Engression: Extrapolation for nonlinear regression?},
  author={Shen, Xinwei and Meinshausen, Nicolai},
  journal={arXiv preprint arXiv:2307.00835},
  year={2023}
}

@article{ning2023elucidating,
  title={Elucidating the exposure bias in diffusion models},
  author={Ning, Mang and Li, Mingxiao and Su, Jianlin and Salah, Albert Ali and Ertugrul, Itir Onal},
  journal={arXiv preprint arXiv:2308.15321},
  year={2023}
}

@article{li2023alleviating,
  title={Alleviating exposure bias in diffusion models through sampling with shifted time steps},
  author={Li, Mingxiao and Qu, Tingyu and Yao, Ruicong and Sun, Wei and Moens, Marie-Francine},
  journal={arXiv preprint arXiv:2305.15583},
  year={2023}
}

@misc{yao2025rollouttrainingmismatch,
  title        = {On the Rollout-Training Mismatch in Modern RL Systems},
  author       = {Yao, Feng and Liu, Liyuan and Zhang, Dinghuai and Dong, Chengyu and Shang, Jingbo and Gao, Jianfeng},
  year         = {2025},
  howpublished = {OPT 2025: 17th Annual Workshop on Optimization for Machine Learning},
  url          = {https://opt-ml.org/papers/2025/paper116.pdf},
}

@misc{deepscaler2025,
  title={DeepScaleR: Surpassing O1-Preview with a 1.5B Model by Scaling RL},
  author={Michael Luo and Sijun Tan and Justin Wong and Xiaoxiang Shi and William Y. Tang and Manan Roongta and Colin Cai and Jeffrey Luo and Li Erran Li and Raluca Ada Popa and Ion Stoica},
  howpublished={\url{https://pretty-radio-b75.notion.site/DeepScaleR-Surpassing-O1-Preview-with-a-1-5B-Model-by-Scaling-RL-19681902c1468005bed8ca303013a4e2}},
  note={Notion Blog},
  year={2025}
}

@article{shang2025rstar2,
  title={rstar2-agent: Agentic reasoning technical report},
  author={Shang, Ning and Liu, Yifei and Zhu, Yi and Zhang, Li Lyna and Xu, Weijiang and Guan, Xinyu and Zhang, Buze and Dong, Bingcheng and Zhou, Xudong and Zhang, Bowen and others},
  journal={arXiv preprint arXiv:2508.20722},
  year={2025}
}

@article{hao2025towards,
  title={Towards Better Generalization via Distributional Input Projection Network},
  author={Hao, Yifan and Lu, Yanxin and Zhang, Hanning and Shen, Xinwei and Zhang, Tong},
  journal={arXiv preprint arXiv:2506.04690},
  year={2025}
}

@article{qi2025defeating,
  title={Defeating the training-inference mismatch via fp16},
  author={Qi, Penghui and Liu, Zichen and Zhou, Xiangxin and Pang, Tianyu and Du, Chao and Lee, Wee Sun and Lin, Min},
  journal={arXiv preprint arXiv:2510.26788},
  year={2025}
}

@article{li2022positive,
  title={Positive-incentive noise},
  author={Li, Xuelong},
  journal={IEEE Transactions on Neural Networks and Learning Systems},
  volume={35},
  number={6},
  pages={8708--8714},
  year={2022},
  publisher={IEEE}
}

@inproceedings{pereira2021multi,
  title={Multi-layer random perturbation training for improving model generalization efficiently},
  author={Pereira, Lis Kanashiro and Taya, Yuki and Kobayashi, Ichiro},
  booktitle={Proceedings of the Fourth BlackboxNLP Workshop on Analyzing and Interpreting Neural Networks for NLP},
  pages={303--310},
  year={2021}
}

@inproceedings{cohen2019rs,
  title={Certified adversarial robustness via randomized smoothing},
  author={Cohen, Jeremy and Rosenfeld, Elan and Kolter, Zico},
  booktitle={international conference on machine learning},
  pages={1310--1320},
  year={2019},
  organization={PMLR}
}

@article{salman2019provably,
  title={Provably robust deep learning via adversarially trained smoothed classifiers},
  author={Salman, Hadi and Li, Jerry and Razenshteyn, Ilya and Zhang, Pengchuan and Zhang, Huan and Bubeck, Sebastien and Yang, Greg},
  journal={Advances in neural information processing systems},
  volume={32},
  year={2019}
}

@inproceedings{lecuyer2019certified,
  title={Certified robustness to adversarial examples with differential privacy},
  author={Lecuyer, Mathias and Atlidakis, Vaggelis and Geambasu, Roxana and Hsu, Daniel and Jana, Suman},
  booktitle={2019 IEEE symposium on security and privacy (SP)},
  pages={656--672},
  year={2019},
  organization={IEEE}
}

@inproceedings{yang2020randomized,
  title={Randomized smoothing of all shapes and sizes},
  author={Yang, Greg and Duan, Tony and Hu, J Edward and Salman, Hadi and Razenshteyn, Ilya and Li, Jerry},
  booktitle={International Conference on Machine Learning},
  pages={10693--10705},
  year={2020},
  organization={PMLR}
}

@article{yu2023noisynn,
  title={Noisynn: Exploring the impact of information entropy change in learning systems},
  author={Yu, Xiaowei and Huang, Zhe and Chen, Minheng and Xue, Yao and Liu, Tianming and Zhu, Dajiang},
  journal={arXiv e-prints},
  pages={arXiv--2309},
  year={2023}
}

@inproceedings{moreno2018forward,
  title={Forward noise adjustment scheme for data augmentation},
  author={Moreno-Barea, Francisco J and Strazzera, Fiammetta and Jerez, Jos{\'e} M and Urda, Daniel and Franco, Leonardo},
  booktitle={2018 IEEE symposium series on computational intelligence (SSCI)},
  pages={728--734},
  year={2018},
  organization={IEEE}
}

@article{he2025nondeterminism,
  author = {Horace He and Thinking Machines Lab},
  title = {Defeating Nondeterminism in LLM Inference},
  journal = {Thinking Machines Lab: Connectionism},
  year = {2025},
  note = {https://thinkingmachines.ai/blog/defeating-nondeterminism-in-llm-inference/},
  doi = {10.64434/tml.20250910}
}

@article{li2025trust,
  title={Trust Region Masking for Long-Horizon LLM Reinforcement Learning},
  author={Li, Yingru and Liu, Jiacai and Xu, Jiawei and Tong, Yuxuan and Li, Ziniu and Wang, Baoxiang},
  journal={arXiv preprint arXiv:2512.23075},
  year={2025}
}

@inproceedings{kwon2023efficient,
  title={Efficient Memory Management for Large Language Model Serving with PagedAttention},
  author={Woosuk Kwon and Zhuohan Li and Siyuan Zhuang and Ying Sheng and Lianmin Zheng and Cody Hao Yu and Joseph E. Gonzalez and Hao Zhang and Ion Stoica},
  booktitle={Proceedings of the ACM SIGOPS 29th Symposium on Operating Systems Principles},
  year={2023}
}

@article{zheng2024sglang,
  title={Sglang: Efficient execution of structured language model programs},
  author={Zheng, Lianmin and Yin, Liangsheng and Xie, Zhiqiang and Sun, Chuyue Livia and Huang, Jeff and Yu, Cody Hao and Cao, Shiyi and Kozyrakis, Christos and Stoica, Ion and Gonzalez, Joseph E and others},
  journal={Advances in neural information processing systems},
  volume={37},
  pages={62557--62583},
  year={2024}
}

@article{fu2025areal,
  title={AReaL: A Large-Scale Asynchronous Reinforcement Learning System for Language Reasoning},
  author={Fu, Wei and Gao, Jiaxuan and Shen, Xujie and Zhu, Chen and Mei, Zhiyu and He, Chuyi and Xu, Shusheng and Wei, Guo and Mei, Jun and Wang, Jiashu and others},
  journal={arXiv preprint arXiv:2505.24298},
  year={2025}
}

@misc{liu2025flashrl,
  title = {FlashRL: 8Bit Rollouts, Full Power RL},
  url = {https://fengyao.notion.site/flash-rl},
  author = {Liu, Liyuan and Yao, Feng and Zhang, Dinghuai and Dong, Chengyu and Shang, Jingbo and Gao, Jianfeng},
  journal = {Feng Yao's Notion},
  year = {2025},
  month = aug,
}

@article{zheng2025group,
  title={Group sequence policy optimization},
  author={Zheng, Chujie and Liu, Shixuan and Li, Mingze and Chen, Xiong-Hui and Yu, Bowen and Gao, Chang and Dang, Kai and Liu, Yuqiong and Men, Rui and Yang, An and others},
  journal={arXiv preprint arXiv:2507.18071},
  year={2025}
}

@inproceedings{schulman2015trust,
  title={Trust region policy optimization},
  author={Schulman, John and Levine, Sergey and Abbeel, Pieter and Jordan, Michael and Moritz, Philipp},
  booktitle={International conference on machine learning},
  pages={1889--1897},
  year={2015},
  organization={PMLR}
}

@software{Kydlicek_Math-Verify_Math_Verification,
author = {Kydlíček, Hynek},
license = {Apache-2.0},
title = {{Math-Verify: Math Verification Library}},
url = {https://github.com/huggingface/math-verify},
version = {0.6.1}
}

@misc{li_rl_collapse,
  author       = {Jiacai, Liu* and Yingru, Li* and   Yuqian, Fu and Jiawei, Wang and Qian Liu and Zhuo Jiang},
  title        = {When Speed Kills Stability: Demystifying {RL} Collapse from the Training-Inference Mismatch},
  howpublished = {\url{https://yingru.notion.site/When-Speed-Kills-Stability-Demystifying-RL-Collapse-from-the-Training-Inference-Mismatch-271211a558b7808d8b12d403fd15edda}},
  year         = {2025},
  note         = {Accessed: 2026-01-28}
}

@article{xue2025simpletir,
  title={Simpletir: End-to-end reinforcement learning for multi-turn tool-integrated reasoning},
  author={Xue, Zhenghai and Zheng, Longtao and Liu, Qian and Li, Yingru and Zheng, Xiaosen and Ma, Zejun and An, Bo},
  journal={arXiv preprint arXiv:2509.02479},
  year={2025}
}

@article{yang2024qwen2,
  title={Qwen2. 5-math technical report: Toward mathematical expert model via self-improvement},
  author={Yang, An and Zhang, Beichen and Hui, Binyuan and Gao, Bofei and Yu, Bowen and Li, Chengpeng and Liu, Dayiheng and Tu, Jianhong and Zhou, Jingren and Lin, Junyang and others},
  journal={arXiv preprint arXiv:2409.12122},
  year={2024}
}

@article{he2024olympiadbench,
  title={Olympiadbench: A challenging benchmark for promoting agi with olympiad-level bilingual multimodal scientific problems},
  author={He, Chaoqun and Luo, Renjie and Bai, Yuzhuo and Hu, Shengding and Thai, Zhen Leng and Shen, Junhao and Hu, Jinyi and Han, Xu and Huang, Yujie and Zhang, Yuxiang and others},
  journal={arXiv preprint arXiv:2402.14008},
  year={2024}
}

@article{lewkowycz2022solving,
  title={Solving quantitative reasoning problems with language models},
  author={Lewkowycz, Aitor and Andreassen, Anders and Dohan, David and Dyer, Ethan and Michalewski, Henryk and Ramasesh, Vinay and Slone, Ambrose and Anil, Cem and Schlag, Imanol and Gutman-Solo, Theo and others},
  journal={Advances in Neural Information Processing Systems},
  volume={35},
  pages={3843--3857},
  year={2022}
}

@article{sheng2024hybridflow,
  title   = {HybridFlow: A Flexible and Efficient RLHF Framework},
  author  = {Guangming Sheng and Chi Zhang and Zilingfeng Ye and Xibin Wu and Wang Zhang and Ru Zhang and Yanghua Peng and Haibin Lin and Chuan Wu},
  year    = {2024},
  journal = {arXiv preprint arXiv: 2409.19256}
}

@misc{zeng2025simplerl,
  title={7B Model and 8K Examples: Emerging Reasoning with Reinforcement Learning is Both Effective and Efficient},
  author={Weihao Zeng and Yuzhen Huang and Wei Liu and Keqing He and Qian Liu and Zejun Ma and Junxian He},
  year={2025},
  howpublished={\url{https://hkust-nlp.notion.site/simplerl-reason}},
  note={Notion Blog}
}

@article{shao2024deepseekmath,
  title={Deepseekmath: Pushing the limits of mathematical reasoning in open language models},
  author={Shao, Zhihong and Wang, Peiyi and Zhu, Qihao and Xu, Runxin and Song, Junxiao and Zhang, Mingchuan and Li, YK and Wu, Y and Guo, Daya},
  journal={arXiv preprint arXiv:2402.03300},
  year={2024}
}

@article{schulman2017proximal,
  title={Proximal policy optimization algorithms},
  author={Schulman, John and Wolski, Filip and Dhariwal, Prafulla and Radford, Alec and Klimov, Oleg},
  journal={arXiv preprint arXiv:1707.06347},
  year={2017}
}

@article{hendrycks2021measuring,
  title={Measuring mathematical problem solving with the math dataset},
  author={Hendrycks, Dan and Burns, Collin and Kadavath, Saurav and Arora, Akul and Basart, Steven and Tang, Eric and Song, Dawn and Steinhardt, Jacob},
  journal={arXiv preprint arXiv:2103.03874},
  year={2021}
}

@misc{zhao2023pytorchfsdpexperiencesscaling,
      title={PyTorch FSDP: Experiences on Scaling Fully Sharded Data Parallel}, 
      author={Yanli Zhao and Andrew Gu and Rohan Varma and Liang Luo and Chien-Chin Huang and Min Xu and Less Wright and Hamid Shojanazeri and Myle Ott and Sam Shleifer and Alban Desmaison and Can Balioglu and Pritam Damania and Bernard Nguyen and Geeta Chauhan and Yuchen Hao and Ajit Mathews and Shen Li},
      year={2023},
      eprint={2304.11277},
      archivePrefix={arXiv},
      primaryClass={cs.DC},
      url={https://arxiv.org/abs/2304.11277}, 
}

@misc{cheng2025revisiting,
  title         = {Revisiting Reinforcement Learning for LLM Reasoning from A Cross-Domain Perspective},
  author        = {Zhoujun Cheng and Shibo Hao and Tianyang Liu and Fan Zhou and Yutao Xie and Feng Yao and Yuexin Bian and Yonghao Zhuang and Nilabjo Dey and Yuheng Zha and Yi Gu and Kun Zhou and Yuqi Wang and Yuan Li and Richard Fan and Jianshu She and Chengqian Gao and Abulhair Saparov and Haonan Li and Taylor W. Killian and Mikhail Yurochkin and Zhengzhong Liu and Eric P. Xing and Zhiting Hu},
  journal       = {arXiv preprint arXiv:2506.14965},
  year          = {2025},
  doi           = {10.48550/arXiv.2506.14965},
  url           = {https://arxiv.org/abs/2506.14965}
}

@misc{openr1,
    title = {Open R1: A fully open reproduction of DeepSeek-R1},
    url = {https://github.com/huggingface/open-r1},
    author = {{Hugging Face}},
    month = {January},
    year = {2025}
}

@misc{yu2025dapoopensourcellmreinforcement,
      title={DAPO: An Open-Source LLM Reinforcement Learning System at Scale}, 
      author={Qiying Yu and Zheng Zhang and Ruofei Zhu and Yufeng Yuan and Xiaochen Zuo and Yu Yue and Weinan Dai and Tiantian Fan and Gaohong Liu and Lingjun Liu and Xin Liu and Haibin Lin and Zhiqi Lin and Bole Ma and Guangming Sheng and Yuxuan Tong and Chi Zhang and Mofan Zhang and Wang Zhang and Hang Zhu and Jinhua Zhu and Jiaze Chen and Jiangjie Chen and Chengyi Wang and Hongli Yu and Yuxuan Song and Xiangpeng Wei and Hao Zhou and Jingjing Liu and Wei-Ying Ma and Ya-Qin Zhang and Lin Yan and Mu Qiao and Yonghui Wu and Mingxuan Wang},
      year={2025},
      eprint={2503.14476},
      archivePrefix={arXiv},
      primaryClass={cs.LG},
      url={https://arxiv.org/abs/2503.14476}, 
}

@article{yang2025qwen3,
  title={Qwen3 technical report},
  author={Yang, An and Li, Anfeng and Yang, Baosong and Zhang, Beichen and Hui, Binyuan and Zheng, Bo and Yu, Bowen and Gao, Chang and Huang, Chengen and Lv, Chenxu and others},
  journal={arXiv preprint arXiv:2505.09388},
  year={2025}
}
\bibliographystyle{tmlr}

\newpage
\appendix

\section{Additional Related Works}
\paragraph{Off-policy}
Off-policy optimization is a challenging setting in reinforcement learning. Trust Region Policy Optimization (TRPO) \citep{schulman2015trust} shows that when the KL-divergence is updated and the behavior policy is bounded, policy improvement is guaranteed. PPO \citep{schulman2017proximal} further simplifies the optimization by proposing a surrogate clipping method. The off-policy issue becomes more severe in LLM reasoning tasks, especially in long CoT and multi-turn settings. To trade-off the bias and balance for importance ratios, sequence-level importance ratios \citep{zheng2025group} and more stringent masking \citep{li2025trust} are developed to stabilize training. Moreover, the quantization and batching issues in the inference engine make the off-policy unavoidable. Recent works \citep{yao2025rollouttrainingmismatch,li_rl_collapse} correct the distribution by multiplying another importance ratio $\pi^{\train}_{\theta_\old}/\pi^{\infer}_{\theta_\old}$ and clipping or masking the outliers. 

Recent systems efforts aim to \emph{eliminate} training-inference mismatch by enforcing bitwise-consistent inference, enabling closer-to on-policy RL in practice \citep{he2025nondeterminism} or using higher precisions \citep{qi2025defeating}.
These directions are promising, but maintaining strict bitwise consistency is computationally inefficient, and can be challenging under common production constraints such as changing execution engines, optional rollout quantization, fully asynchronous rollouts, and rapidly evolving kernels.
Our work is complementary: rather than assuming mismatch can be removed everywhere, we propose an algorithmic correction that remains robust when such mismatch persists, providing a unified importance-ratio formulation that tolerates heterogeneous sources of off-policy deviation. Although those methods are effective, they only focus on a specific problem and current RL training needs to combine the techniques and tune the parameters separately. The algorithm becomes even more complicated in fully-asynchronous settings \citep{fu2025areal} since there are three importance ratios. Hence, our work develops a unified importance ratio instead of dividing the importance ratio into several parts to solve general off-policy problems.

\section{Technical Lemmas}

\begin{lemma}[Stam's inequality]\label{lem:stam}
For any two independent random variables $(X, Y)$, we have
\begin{equation*}
    I(X + Y) \preceq \min \{ I(X) , I(Y) \},
\end{equation*}
where $I(\cdot)$ is defined as the Fisher information matrix of the corresponding random variable.    
\end{lemma}

\begin{lemma}\label{lem:kernel}
 If we assume $\delta \sim \mathcal{N}(0, \sigma^2 I_d)$, and define $\tilde{\pi}_\theta(a | x) := \E_\delta \pi(a | x + \delta) $,
 there will be
 \begin{equation*}
     \nabla_x \ln \tilde{\pi}_\theta (a | x) = \frac{1}{\sigma^2} \E_{\delta \sim q_\theta(\cdot| a, x)} [ \delta],
 \end{equation*}
 where
 \begin{equation*}
     q_\theta(\delta | a, x) = \frac{\pi_\theta(a | x+\delta) \phi(\delta)}{\tilde{\pi}_\theta(a | x)}, \quad \phi(\delta) \sim \mathcal{N}(0, \sigma^2 I_d).
 \end{equation*}
\end{lemma}

\begin{proof}
First, we could calculate the gradient of $\tilde{\pi}_\theta (a | x)$ with respect to $x$ as:
\begin{align*}
 \nabla_x  \tilde{\pi}_\theta (a | x) &:= \int \nabla_x \pi_\theta(a | x + \delta) \phi(\delta) d \delta \\
 &= \int \nabla_\delta \pi_\theta(a | x + \delta) \phi(\delta) d \delta \\
 &= - \int \pi_\theta (a | x + \delta) \nabla_\delta \phi(\delta) d \delta \\
 &= \int \frac{\delta}{\sigma^2} \pi_\theta (a | x + \delta) \phi(\delta) d \delta \\
 &= \frac{\tilde{\pi}_\theta(a | x)}{\sigma^2} \int \delta \frac{\pi_\theta(a | x+\delta) \phi(\delta)}{\tilde{\pi}_\theta(a | x)} d \delta \\
 &= \frac{\tilde{\pi}_\theta(a | x)}{\sigma^2} \E_{\delta \sim q_\theta(\cdot | a, x)} \delta,
\end{align*}
where the third and fourth equality use integration by parts (Stein’s identity for a Gaussian) to move the derivative from the complex term  onto the Gaussian density $\phi(\delta)$, whose gradient has a closed form and whose boundary term vanishes as $\|\delta\|\to\infty$.

Then take it into the expression of $\nabla_x \ln \tilde{\pi}_\theta (a | x)$, we have
\begin{equation*}
   \nabla_x \ln \tilde{\pi}_\theta (a | x) = \frac{1}{\sigma^2} \E_{\delta \sim q_\theta(\cdot| a, x)} [ \delta] .
\end{equation*}
\end{proof}

\section{Formal Statement and Proof for Theorem~\ref{thm:mismatch_informal}}\label{pf:mismatch}

\begin{condition}\label{cond:bias}
For any prompt $x$ and action $a$,  we have
\begin{equation*}
  \| \E_{\delta \sim q(\cdot | a, x)} [\delta] \|_2^2 \le Cd \sigma^2,
\end{equation*}
where $C > 0$ is a constant and $q_\theta(\cdot| a, x)$ is the posterior distribution of $\delta$ conditioned on $(a, x)$. 
\end{condition}

\begin{condition}\label{cond:distribution}
 There exist a constant $0 < \alpha \le 1$ such that for any $a$ and $x$, we have
 \begin{equation*}
   \alpha \cdot \tilde{\pi}_\theta (a | x) \le \pi_\theta (a | x) .
 \end{equation*}
\end{condition}

Condition~\ref{cond:bias} requires that the squared $\ell_2$ norm of the posterior mean of the perturbation is bounded by $\mathcal{O}(d \sigma^2)$, where $d$ is the input dimension and $\sigma^2$ is the variance of the Gaussian prior $\delta \sim \mathcal{N}(0, \sigma^2 I_d)$. This scaling is natural and aligns with standard concentration behavior in high-dimensional Gaussian settings.
Condition~\ref{cond:distribution} ensures that the perturbed policy $\tilde{\pi}_\theta$ is uniformly dominated by the original policy $\pi_\theta$ up to a constant factor $\alpha$. Intuitively, this condition guarantees that the injected perturbation is well controlled and does not overly distort the action distribution, which could otherwise lead to instability during training. In the limiting case $\delta \to 0$, the condition is trivially satisfied with $\alpha = 1$.

To be specific, we can obtain:
\begin{theorem}[Formal Version of Theorem \ref{thm:mismatch_informal}]\label{thm:mismatch}
 With Condition~\ref{cond:bias} and Condition~\ref{cond:distribution}, we have
\begin{equation}\label{eq:kl_mismatch2}
   \E_{x \sim d_0,\zeta} \mathrm{KL} \left( \tilde{\pi}_{\theta_\old} (a |x) \parallel \pi_{\theta_\old}^\infer(a|x)  \right) \le  - \ln \alpha + \frac{ C d \E \| \zeta \|_2^2}{2 \sigma^2}.
\end{equation}
\end{theorem}
The result in Theorem~\ref{thm:mismatch} shows that with properly learned value of $\sigma$, the training-inference mismatch can be controlled effectively during the training process.
\begin{proof}
In ALP, we could obtain the KL distance as:
\begin{align*}
   & \quad \E_{ \zeta} \E_{x \sim d_0} \mathrm{KL} \left( \tilde{\pi}_{\theta_\old} (a |x) \parallel \pi_{\theta_\old}^\infer(a|x)  \right)\\
   &= \E_\zeta \E_{x \sim d_0} \sum_a \tilde{\pi}_{\theta_\old} (a | x) \ln \left( \frac{\tilde{\pi}_{\theta_\old}(a | x)}{\pi_{\theta_\old}(a | x + \zeta)} \right)\\
   &\le \E_\zeta \E_{x \sim d_0} \sum_a \tilde{\pi}_{\theta_\old} (a | x) \ln \left( \frac{\tilde{\pi}_{\theta_\old}(a | x)}{ \alpha \tilde{\pi}_{\theta_\old}(a | x + \zeta)} \right)\\
   &\approx - \ln \alpha + \frac{1}{2} \E_{\zeta} \zeta^T \E_{x \sim d_0, a \sim \tilde{\pi}_{\theta_\old}(\cdot | x)} \nabla_x \ln \tilde{\pi}_{\theta_\old}(a | x) \nabla_x \ln \tilde{\pi}_{\theta_\old}(a | x)^T \zeta \\
   &\le - \ln \alpha + \frac{1}{2} \E \| \zeta \|_2^2 \E_{x \sim d_0, a \sim \tilde{\pi}_{\theta_\old}(\cdot | x)} \left\| \nabla_x \ln \tilde{\pi}_{\theta_\old}(a | x) \right\|_2^2 \\
   &= - \ln \alpha +  \frac{1}{2 \sigma^4} \E \| \zeta \|_2^2 \E_{x \sim d_0, a \sim \tilde{\pi}_{\theta_\old}(\cdot | x)} \left\| \E_{\delta \sim q_\theta(\cdot| a, x) } \delta \right\|_2^2 \\
   &\le - \ln \alpha + \frac{ C d \E \| \zeta \|_2^2}{2 \sigma^2}.
\end{align*}
where the first inequality is from Condition~\ref{cond:distribution}, the equality on the sixth line is from Lemma~\ref{lem:kernel}, and the last inequality is from Condition~\ref{cond:bias}.
So we could bound the KL distance as
\begin{equation}
   \E_{x \sim d_0,\zeta} \mathrm{KL} \left( \tilde{\pi}_{\theta_\old} (a |x) \parallel \pi_{\theta_\old}^\infer(a|x)  \right) \le  - \ln \alpha + \frac{ C d \E \| \zeta \|_2^2}{2 \sigma^2}.
\end{equation}

\end{proof}

\section{Proof for Theorem~\ref{thm:smooth_informal}}\label{pf:smooth}

\begin{theorem}\label{thm:smooth}
Define $\mathcal{I}(x, \theta) = \|\nabla^2_\theta J(\theta)\|_2,\quad \tilde{\mathcal{I}}(x, \theta) = \|\nabla^2_\theta J_{\mathrm{ALP}}(\theta, \sigma)\|_2$. If there exists a smallest constant $0<c<1$ such that
\begin{equation*}
    \int \bm{1} \left( \mathcal{I}(x, \theta) \ge c/2 \cdot \max_x \mathcal{I}(x, \theta)\right) dx < + \infty,
\end{equation*}
 with some proper distribution $\delta \sim \mathcal{N}(0, \sigma^2 I)$, we have
\begin{equation*}
\sup_x \tilde{\mathcal{I}}(x, \theta) \le c \cdot \sup_x \mathcal{I}(x, \theta).
\end{equation*}
\end{theorem}

To start with, we denote $b:= \max_x \mathcal{I}(x,\theta)$, and there exists $c\in(0,2)$ such that 
\begin{equation*}
    \int \bm{1} ( \mathcal{I}(x, \theta) \ge c/2 \cdot b ) \cdot dx = C < + \infty.
\end{equation*}
Then for any $\epsilon > 0$, we can choose proper std value $\sigma$, such that for any r satisfying that the volume of $r$-radius ball is no larger than $C$, i.e, $\mu\{ \| \delta \|_2 \le r  \} \le C$, we have
\begin{equation}\label{eq:cond}
    \mathcal{P} \left( \| \delta \|_2 \le r \right) \le \epsilon.
\end{equation}
Then for any $x, a$, we have
\begin{align*}
   & \quad \left\| \E_\delta \nabla^2_\theta  \E_{a \sim \pi_{\theta_\old}^\infer(a|x)} \left[ \frac{\pi_\theta (a|x + \delta)}{\pi_{\theta_\old}^\infer(a|x)} \cdot \hat{A}(a) \right]  \right\|_2 \le \E_\delta \left\|  \nabla^2_\theta  \E_{a \sim \pi_{\theta_\old}^\infer(a|x)} \left[ \frac{\pi_\theta (a|x + \delta)}{\pi_{\theta_\old}^\infer(a|x)} \cdot \hat{A}(a) \right]  \right\|_2 \\
    &= \int_{\mathcal{I}(x+\delta, \theta) \ge bc /2} \mathcal{I}(x+\delta, \theta)  \phi(\delta) d \delta 
    + \int_{\mathcal{I}(x+\delta, \theta) < bc /2} \mathcal{I}(x+\delta, \theta)  \phi(\delta) d \delta  \\
    &\le b  \cdot \int_{\mathcal{I}(x+\delta, \theta) \ge bc /2}  \phi(\delta) d \delta  + \frac{bc}{2}  \\
    &\le (\frac{c}{2} + \epsilon) \cdot b,
\end{align*}
where the last inequality is from the fact that the measure of the set satisfying $\mathcal{I}(x, \theta) \ge c/2 \cdot b$ equals $C$, and for a Gaussian distribution, the probability on such a set is smaller than the probability on a zero-center ball with volume $C$. Combining this fact with the assumption in \eqref{eq:cond}, we derive the last inequality.

While $\epsilon \le \frac{c}{2}$, we can finish the proof.

\section{Additional Experimental Details}\label{sec:Additional Experimental Details}
\paragraph{Single-turn details.}
The single-turn training set is constructed from the math subset of Guru RL-92k and a $75$K OpenR1 subset. Guru contains $91.9$K high-quality samples spanning six reasoning-intensive domains, while OpenR1 is derived from a larger corpus of math problems with two to four reasoning traces generated by DeepSeek R1. For prompts in the single-turn task, we filter OpenR1 examples by their average score in $(0.2,0.8)$, and then deduplicate the merged corpus by user-prompt content. Prompts are converted into a system-plus-user format with a fixed system instruction that requests a final \verb|\boxed{}| answer.
Rollouts are generated with vLLM\citep{kwon2023efficient}, while training is performed with Fully Sharded Data Parallel (FSDP)\citep{zhao2023pytorchfsdpexperiencesscaling}. We set \texttt{max\_prompt\_length=2048} and \texttt{max\_response\_length=2048}. We use dual-clip PPO with \texttt{clip\_ratio\_c=10.0} throughout. Each step samples 512 prompts and generates $n=8$ responses per prompt. The policy is optimized with AdamW (learning rate $10^{-6}$; weight decay $0.01$). Rewards are computed via rule-based math verification using the \texttt{math-verify} Python package, by extracting the final \verb|\boxed{}| answer and matching it to the ground truth.
We set the KL loss coefficient to $0.001$ and entropy loss coefficient to $0.001$. The models are trained with $8$ H100 GPUs and $4$ A100 GPUs. For the additional Qwen3-4B validation in Table~\ref{tab:single_turn_qwen3_4b}, we keep the same training and evaluation protocol and only change the backbone model.

We further provide the clipping or masking details for the baselines in the following table.
\begin{table}[ht]
\small
\begin{center}
\begin{tabular}{llll}
\hline
\textbf{Baselines} & $\pi_{\theta}^\train/ \pi_{\theta_\old}^\train$ & $\pi_{\theta_\old}^\train/ \pi_{\theta_\old}^\infer$ & $\pi_{\theta}^\train/ \pi_{\theta_\old}^\infer$\\
\hline
GRPO & token & None & None \\
Token-MIS & token & seq & None \\
Seq-MIS & seq  & seq & None\\
Seq-Bypass & None  & None & seq\\
\hline
\end{tabular}
\end{center}
\caption{Settings of parameters of importance ratios}
\label{t:Settings of parameters of importance ratios}
\end{table}
For the clipping or masking threshold parameters, they are chosen to achieve stable performance for different settings respectively. Specifically, for token-level importance ratio, the lower clipping ratio is $0.2$, and the higher clipping ratio is $0.2$ or $0.28$; for sequence-level importance ratio, the lower clipping ratio is $0.5$, and the higher one is $3.0$. Besides, for MIS, the masking threshold is $2.0$.

\paragraph{Multi-turn details.}
For multi-turn experiments, we use Tool-Integrated Reasoning (TIR) with a Python interpreter, allowing the model to execute code for intermediate verification and symbolic computation. We adapt the SimpleTIR codebase and mask void turns in the loss, where void turns are turns without tool calls and without final answers. The training data merges Math3-5 examples from SimpleRL, Deepscaler, and rStar2-Agent. Training starts from Qwen2.5-7B-base under the zero-RL setting, with rollout batch size $512$, mini-update size $32$, maximum response length $16$K, and at most $5$ code-execution turns per episode.


\begin{figure}[t]
    \centering
    \includegraphics[width=0.95\linewidth]{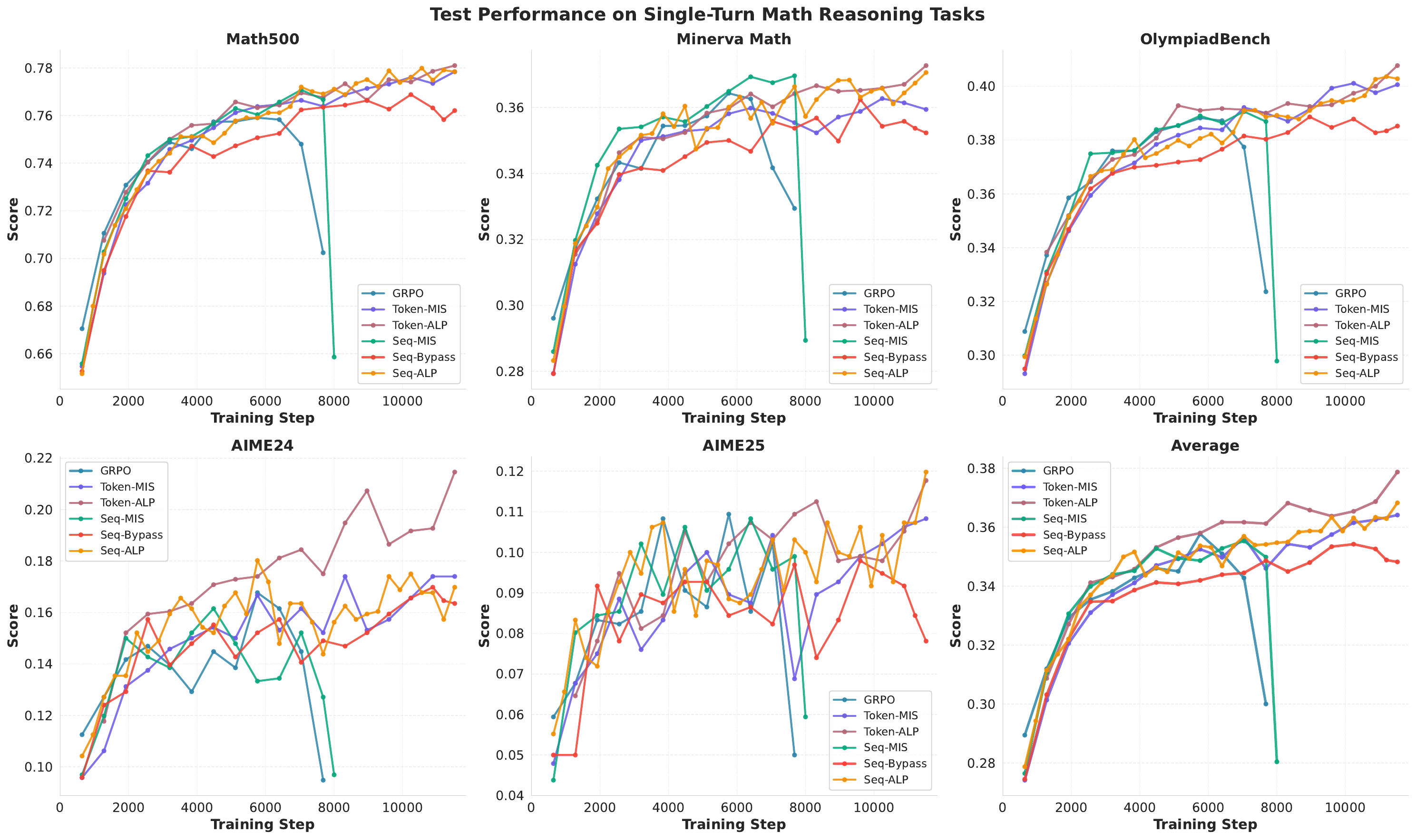}
    \caption{Single-turn evaluation on the test dataset. We compare final-task performance of GRPO, MIS, Bypass, and ALP, highlighting that stable optimization dynamics translate into improved reward/generalization.}
    \label{fig:single_turn_test}
\end{figure}

\begin{figure}[t]
    \centering
    \includegraphics[width=0.95\linewidth]{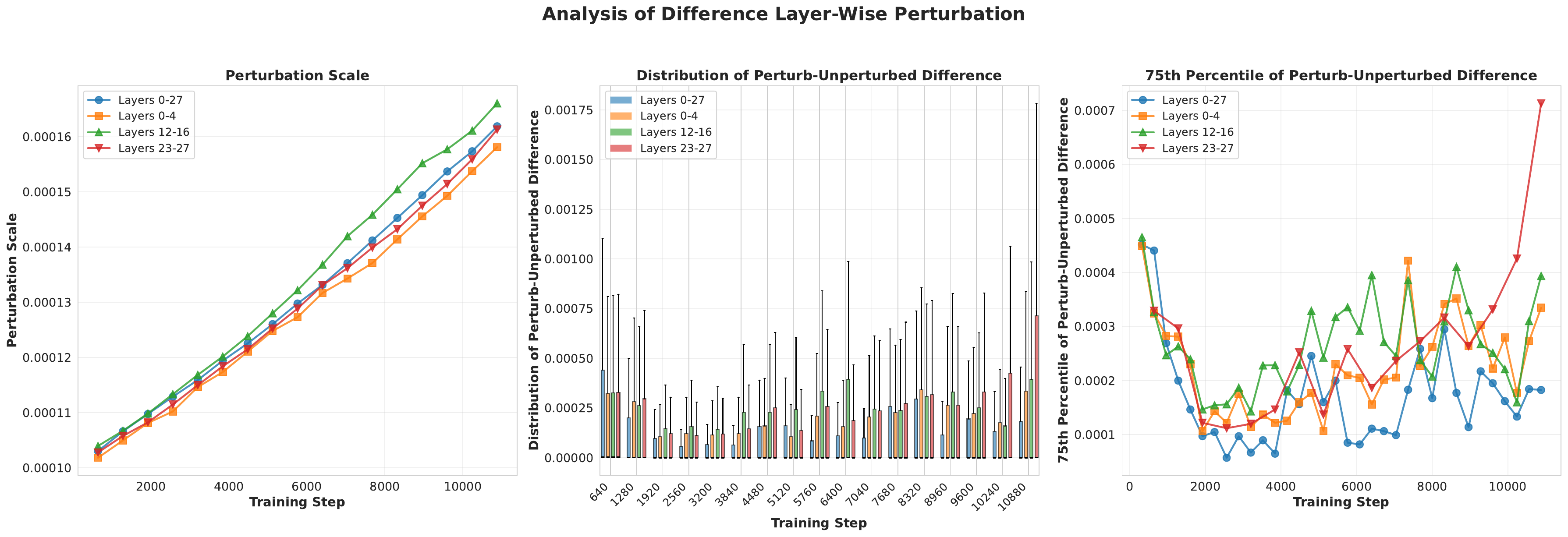}
    \caption{Layer-wise perturbation diagnostics across target choices: (left) mean adaptive noise scale $\sigma$, (middle) distribution of perturbation-induced probability differences $|\Delta p|$, and (right) tail behavior measured by the 75th percentile of $|\Delta p|$ over training steps.}
    \label{fig:ablation_diagnostics}
\end{figure}

\subsection{Hyperparameter sensitivity}\label{sec:ablation_hparam}

We study the sensitivity of ALP to two key hyperparameters: the initial Gaussian noise standard deviation (std) and the learning rate (lr) for the perturbation parameters. We evaluate the resulting policies on Math500, Minerva Math, OlympiadBench, AIME24, and AIME25, and report the average@32 score.

\begin{figure}[t]
    \centering
    \includegraphics[width=0.95\linewidth]{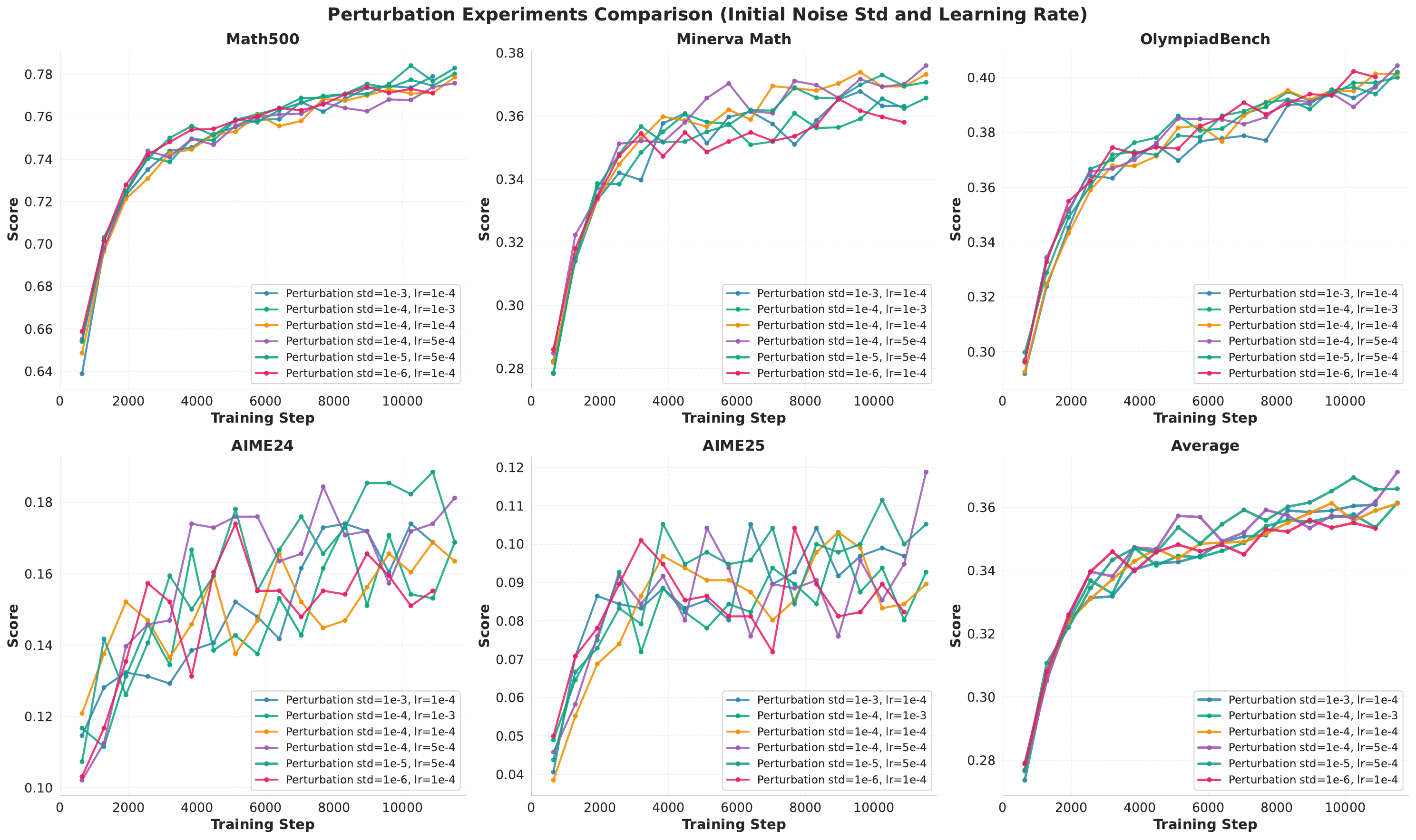}
    \caption{Hyperparameter sensitivity in the single-turn setting. We vary the initial Gaussian noise standard deviation (std) and its learning rate (lr), and report average@32 across Math500, Minerva Math, OlympiadBench, AIME24, and AIME25.}
    \label{fig:ablation_hparam}
\end{figure}

Figure~\ref{fig:ablation_hparam} shows that performance curves are tightly clustered across a broad range of (std, lr) values and improve monotonically over training, suggesting that ALP is not overly sensitive to these hyperparameters. Based on the final-stage performance, we use std $= 1\times 10^{-4}$ and lr $= 5\times 10^{-4}$ as the default configuration in subsequent experiments.

\begin{figure}[t]
  \centering
  \includegraphics[width=0.95\linewidth]{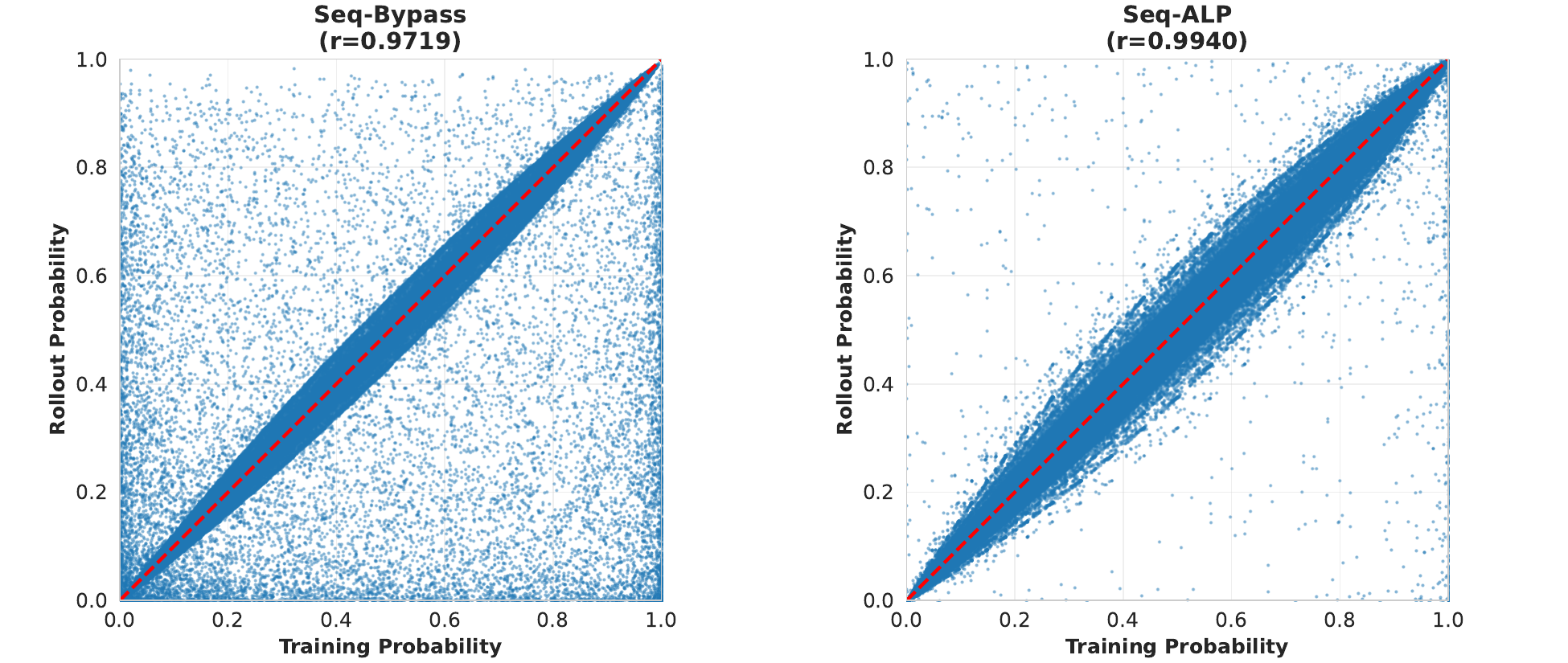}
  \caption{
  Token-level training-inference mismatch at step 140 in the multi-turn setting.
  Each point is one response token, with x-axis denoting the training probability
  $\pi_{\theta_{\mathrm{old}}}^{\mathrm{train}}$ and y-axis denoting the rollout-time token probability
  $\pi_{\theta_{\mathrm{old}}}^{\mathrm{infer}}$.
  The red dashed diagonal is $y=x$; tighter concentration around this line indicates smaller
  training-inference mismatch.
  The Pearson correlation coefficient $r$ is reported in each panel title.
  }
  \label{fig:multi_turn_train_rollout_step140}
\end{figure}

\end{document}